%% file: ms.tex
\documentclass[11pt,a4paper]{article}
\usepackage{authblk}

\usepackage{graphicx}

\usepackage[utf8]{inputenc}

\usepackage[]{color}
\usepackage{colortbl}
\usepackage{makecell}
\usepackage{multirow}

\usepackage{natbib}

\usepackage{microtype}
\usepackage{graphicx}
\usepackage{booktabs} % for professional tables

\usepackage{hyperref}

\definecolor{cellbackground}{rgb}{0.95, 0.96, 0.98}

\usepackage{graphbox}  % vertical alignment option for figures

\usepackage{tikz}
\usetikzlibrary{hobby} 

\usepackage{caption}
\usepackage{subcaption}

\def\arrvline{\hfil\kern\arraycolsep\vline\kern-\arraycolsep\hfilneg}
\usepackage{placeins}

\usepackage{epstopdf}

\usepackage{amsfonts}         
\usepackage{amsmath}
\usepackage{amssymb}
\usepackage{amsthm}
\usepackage{bbm}

\newtheorem{theorem}{Theorem}

\DeclareMathOperator*{\argmin}{arg\,min}

\newcommand{\R}{\mathbb{R}}

\newcommand{\HH}{\mathcal{H}}
\newcommand{\X}{\mathcal{X}}
\newcommand{\x}{\mathbf{x}}
\newcommand{\z}{\mathbf{z}}
\newcommand{\uu}{\mathbf{u}}

\usepackage{bm}  % for bold math symbols (mostly greek)

\makeatletter
\newcommand*{\centerfloat}{%
  \parindent \z@
  \leftskip \z@ \@plus 1fil \@minus \textwidth
  \rightskip\leftskip
  \parfillskip \z@skip}
\makeatother

\usepackage{algorithm,algorithmic}

\usepackage{relsize}

\newcommand{\PKM}{{\sc SPKM}} % my proposal

\urldef\bcwurl\url{https://archive.ics.uci.edu/ml/datasets/Breast+Cancer+Wisconsin+%28Prognostic%29}
\urldef\paurl\url{https://zenodo.org/record/3464542#.XmEEVelS8UE}
\urldef\bostonurl\url{https://scikit-learn.org/stable/modules/generated/sklearn.datasets.load_boston.html#sklearn.datasets.load_boston}

\begin{document}

\title{Learning primal-dual sparse kernel machines\thanks{This work has been in part funded Academy of Finland grants 310107 (MACOME) and 334790 (MAGITICS).}
}

% ===================================================================================
% authors and institutions

% https://tex.stackexchange.com/questions/214404/add-affiliations-to-the-authors-name-in-the-article-class

\author[1]{Riikka Huusari}
\author[2]{Sahely Bhadra}
\author[3]{Cécile Capponi}
\author[3]{Hachem Kadri}
\author[1]{Juho Rousu}
\affil[1]{Helsinki Institute for Information Technology HIIT, Department of Computer Science, Aalto University, Espoo, Finland}
\affil[2]{Indian Institute of Technology, Palakkad, Kerala, India}
\affil[3]{QARMA, LIS, Aix-Marseille University, Marseille, France}

\setcounter{Maxaffil}{0}
\renewcommand\Affilfont{\itshape\small}

\date{}

\maketitle

\begin{abstract}
    Traditionally, kernel methods rely on the representer theorem which states that the solution to a learning problem is obtained as a linear combination of the data mapped into the reproducing kernel Hilbert space (RKHS). While elegant from theoretical point of view, the theorem is prohibitive for algorithms' scalability to large datasets, and the interpretability of the learned function. In this paper, instead of using the traditional representer theorem, we propose to search for a solution in RKHS that has a pre-image decomposition in the original data space, where the elements don't necessarily correspond to the elements in the training set. Our gradient-based optimisation method then hinges on optimising over possibly sparse elements in the input space, and enables us to obtain a kernel-based model with both primal and dual sparsity. We give theoretical justification on the proposed method's generalization ability via a Rademacher bound. Our experiments demonstrate a better scalability and interpretability with accuracy on par with the traditional kernel-based models. 
\paragraph{keywords}{Kernel methods; Sparsity; Latent variable models}
\end{abstract}

\input{intro}

\section{Supervised learning with kernel pre-images}\label{sec:method}

\input{model}

\section{Experiments}\label{sec:exp}
\input{exp}

\input{conclusion}

\paragraph{Acknowledgements}
The authors acknowledge the computational resources provided by the Aalto Science-IT project.
This work has been in part funded Academy of Finland grants 310107 (MACOME) and 334790 (MAGITICS)

\bibliographystyle{spbasic}    

\bibliography{mybib}

\FloatBarrier
\clearpage

\input{appendix}

\end{document}

%% file: intro.tex
\section{Introduction}

Kernel methods in various machine learning contexts have a long and successful history~\citep{hofmann2008kernel}. They are theoretically well-founded approaches for efficiently learning non-linear functions thanks to the so-called kernel trick, making it possible to efficiently calculate inner products in a high-dimensional (even infinite-dimensional) reproducing kernel Hilbert space (RKHS) without explicitly computing the associated feature map.

While very successful in many applications and theoretically well-founded, kernel methods have their drawbacks. The issue most often mentioned is the dependency on the $n\times n$ kernel matrix (where $n$ is the number of training samples), giving rise to ${\Omega}(n^2)$ factor in memory and time complexities of the algorithms. While this bottleneck can in part be worked around with techniques like kernel approximation -- such as the Nyström method~\citep{williams2001using,drineas2005nystrom} and Random Fourier features (RFF)~\citep{rahimi2008random} -- as well as GPU optimization \citep{meanti2020kernel}, kernel methods do not scale up to very large datasets as well as one could hope for. 

Another drawback of traditional kernel machines is the limited interpretability. Kernel-based methods search for a solution to a learning problem in reproducing kernel Hilbert spaces (RKHS), that often in practice can be infinite-dimensional. With these solutions there is a lack of explicit representation of the learned model in the original feature space (i.e. lack of exact pre-images) and the lack of sparsity of the models: the models generally depend on all the input variables, and generally a preprocessing step prior learning is needed to select features \citep{guyon2003introduction}. The interpretability bottleneck is not alleviated by the commonly used kernel approximation approaches such as Nyström and RFF methods, as they produce dense models in terms of the input variables. 

In this paper, we propose an alternative for addressing these concerns: a kernel-based non-linear latent variable model, where the primal sparsity of the basis vectors---the pre-images of non-linear latent variables---can be controlled. The model does not require the $n\times n$ kernel matrix, which alleviates the scalability bottleneck. The learned decision function is represented as a combination of a small number of basis vectors in the original feature space, that represent the different aspects of the data. Thus, our model obtains both primal (low number of non-zeros in basis vectors) and dual (low number of basis vectors) sparsity.
This sparse representation facilitates interpretation of the learned model: one can directly examine which features are being used by the model, and the non-zero feature combinations in each basis vector further facilitate the interpretation.

Our main conceptual starting point is, 
instead of using the traditional representer theorem to find the dual solutions to the learning problems by optimising $n$ dual variables, to perform the optimisation in the original input space by making use of the pre-image of the kernel, which allows us to learn a small number of basis vectors.  This allows the method to scale better to large datasets. This technique was previously successfully used by \citet{uurtio2019large} in the context of canonical correlation analysis, and here we extend it to supervised learning.
Indeed, while the time complexity of traditional kernel methods generally have a quadratic dependency on the number of training points, the complexity of our (full) gradient update is $\mathcal{O}(dnR)$ for commonly used loss functions and kernels, in which $d$ is the dimensionality of the data and $R\ll n$ is the number of basis vectors.

In addition, compared to commonly used kernel approximation methods which are generally unsupervised techniques used as preprocessing methods prior to learning, our method learns an optimal latent representation for the supervised learning task. Our method is very general in the sense, that it can be applied with any differentiable loss function, including squared loss, log-loss and squared hinge loss.

We next move onto describing our method and framework (Section~\ref{sec:method}), which we illustrate in Section~\ref{sec:illustration}.
We show the experimental performance in Section~\ref{sec:exp}, before discussing the related works~(Section~\ref{sec:related}) and concluding~(Section~\ref{sec:concl}).

\paragraph{Notation.} Throughout the paper, we denote scalars, vectors and matrices as $a$, $\mathbf{a}$ and $\mathbf{A}$, respectively.
We consider supervised learning problems with labeled data as $S = \{(\mathbf{x}_i,y_i)\}_{i=1}^n$, in which $\mathbf{x}_i\in\R^d$ and $y_i\in\{-1, +1\}$ for classification and $y_i\in\R$ for regression. 
We also denote the data in matrix form as $\mathbf{X} = [\mathbf{x}_i]_{i=1}^n \in \R^{n\times d}$, and gather all the labels in vector $\mathbf{y}\in\R^n$.

%% file: model.tex
In this paper, our focus is to work with kernel functions $k:\X\times\X\rightarrow\R$, which are symmetric and positive semi-definite (PSD), intuitively seen as measures of similarity between two data samples. The power of the kernel methods hinges on the so-called kernel trick: a kernel can be written as an inner product of the data samples mapped into some feature space, that is, $k(\x,\z)=\langle \phi(\x), \phi(\z) \rangle$. Here $\phi:\X\rightarrow\HH$ is the (non-unique) function mapping data samples from the original data space to a reproducing kernel Hilbert space (RKHS) $\HH$ associated with the kernel. 
Due to this kernel trick, it is possible to non-expensively map the data into very high-dimensional spaces (even infinite-dimensional spaces), in which linear models can then be applied. 

The RKHS of a kernel consists of functions $f:\X\rightarrow\R$. Thus when solving supervised learning problems with kernel methods, it is standard to assume that the function $f$ to be learned belongs to the RKHS. 
Traditionally, kernel methods rely on the representer theorem for the solutions. The theorem states, that the optimal solution to a given learning problem, $w\in\HH$, can be written as linear combination of the training data mapped to the RKHS, or, $w=\sum_i\alpha_i\phi(\x_i)$. In other words, the solution $w$ lies within the span of the images of the training data points in the RKHS (Figure \ref{fig:preimg_spaces}). 

In this paper, we propose to search for a solution $f\in\HH$, that can be characterised in terms of a small number latent factors $\phi(\uu)$, all of which have a pre-images in the original data space, $\uu =\phi^{-1}(w)\in\mathcal{X}$, the basis vectors (note abuse of the term basis; we do not assume a linearly independent spanning set).
Doing this, we restrict ourselves into solutions lying in the image of original data space in the RKHS, instead of considering the full space. %
Yet, unlike in conventional kernel methods, the solution $f$ is not restricted to the span of the training data, since the images of the basis vectors $\phi(\uu)$ can lie outside the span (see Figure \ref{fig:preimg_spaces}).

This results in a non-linear latent variable model, parameterised by the basis vectors $\uu$: we can recover solutions in the original data space that are used non-linearly in the decision function.
Analogously to the dual variables in classical kernel machines, we assign every latent factor $\phi(\uu_r)$ an importance weight $c_r$.
As a whole, this amounts to learning a function $f\in\HH$, having the form 
\begin{equation}\label{eq:preimg_model}
  f(\cdot)=\sum_{r=1}^R c_r \phi(\uu_r) 
\end{equation}
with $R\ll n$, $c_r\in\R$ and $\uu_r\in\mathcal{X}$. We can now write $f(\x_i) = \sum_r c_rk(\x_i, \uu_r) = k(\x_i, \mathbf{U})\mathbf{c}$, where $k(\x_i, \mathbf{U})$ stands for the $1\times R$ matrix of kernel evaluations between $\x_i$ and every $\uu_r$, and the $c_r$ are stacked in $\mathbf{c}\in\R^R$.

Compared to the traditional representer theorem, the $\uu_r$ are not resctricted to match the data samples in the training set, and also much smaller amount of these basis vectors are used to represent $f$. In particular, a connection to traditional models such as SVM or kernel ridge regression (KRR) can be readily seen: take all training examples as the basis vectors  and interpret the $c_r$'s as the dual variables. Yet with our model, as we illustrate in Section~\ref{sec:illustration}, $R$ can be very small compared to the number of training examples.
Additionally, the fact that we have an explicit representation of the pre-images of the latent factors $\phi(\uu_r)$ in terms of the basis vectors $\uu_r$ lets us control the primal sparsity of the latent factors, thus achieving a non-linear model (through a non-linear kernel) with primal (few non-zero elements in basis vectors) and dual (small number of basis vectors) sparsity. It is also straight-forward to extend the model by group or structured sparsity \citep{huang2010benefit,jia2010factorized}, to allow incorporation of domain knowledge in the model learning.

\input{pre-img_img}

% ---------------------------------------------------------------------------------------------
\subsection{The optimisation procedure}\label{sec:algo}

In this section we go through the practical aspects of learning the $f$ in our framework. In the following discussion we consider $\mathcal{X}=\R^d$. 

With the model (\ref{eq:preimg_model}), we consider an iterative two-step learning approach consisting of learning both the basis vectors $\mathbf{U} =[\mathbf{u}_r]_{r=1}^R$ and the multipliers $c_r$. 
To this end, we consider the optimisation problem
\begin{equation}\label{eq:opt_probl}
    \argmin_{\mathbf{U}, \mathbf{c}} \mathcal{L}\left( k(\mathbf{X}, \mathbf{U})\mathbf{c} , \mathbf{y} \right) + \gamma \Theta(\mathbf{U}) + \lambda \|\mathbf{c}\|_a,
\end{equation}
in which $\mathcal{L}$ is a loss function, $\Theta$ a regulariser, $a=\{1,2\}$, and $k(\mathbf{X}, \mathbf{U})$ stands for the $n\times R$ matrix of kernel function evaluations between the data and basis vectors; $f(\x_i) = k(\x_i, \mathbf{U})\mathbf{c}$.
We optimise this problem in iterative fashion.

As a first step, we optimise the the basis vectors from the sub-problem
\begin{equation}\label{eq:u_probl}
    \argmin_{\mathbf{U}} \mathcal{L}\left( k(\mathbf{X}, \mathbf{U})\mathbf{c} , \mathbf{y} \right) + \gamma \Theta(\mathbf{U}).
\end{equation}
We propose to optimise the basis vectors $\mathbf{U} =[\mathbf{u}_r]_{r=1}^R$ with gradient descent. 
Thus $\mathcal{L}$ in~(\ref{eq:u_probl}) can be any differentiable loss function (squared loss, squared hinge, log loss, exp loss etc.). The method can be extended to piecewise differentiable convex losses (such as hinge loss) by subgradient approaches; however, we leave this as future work.

It is easy to calculate the derivatives w.r.t. $\mathbf{u}_r$ for many popular kernels such as the Gaussian/RBF kernel and polynomial kernels, as well as commonly used differentiable loss functions (see supplementary material for examples). 
By not needing to calculate the kernel matrices, we gain in computational complexity w.r.t. the number of samples, $n$. The complexity of calculating the derivative w.r.t. $\mathbf{U}$ is $\mathcal{O}(dnR)$ for typical kernels and loss functions.
Also, with our method there is no need to keep a kernel matrix in memory, as each iteration calculates kernel evaluations $k(\x_i,\uu_r)$ on the fly for the current $\uu_r$ vectors.

Instead of using the traditional regulariser $\|f\|_\HH$, we in~(\ref{eq:u_probl}) regularise the $\uu_r$  and $\mathbf{c}$ independently.
In order to obtain primal sparsity, we consider $l_1$-norm regularisation, which results in using proximal gradient methods or $l_1$ ball projection~\citep{duchi2008efficient} as part of the optimisation process. We note here in passing that it is easy to incorporate group sprase or structured sparse norm to allow leveraging prior information of the application domain. We leave this extension to future work. We note that, despite the lack of traditional regulariser, we are able to obtain generalisation guarantees with our algorithm (see Section~\ref{sec:theory}).

In order to solve (\ref{eq:opt_probl}) for $\mathbf{U}$, the $c_r$ are initialised based on the user's initial guess or prior knowledge about the data: for example in classification by how many basis vectors from positive and negative classes are searched for and their relative importances.
As a second step in our optimisation procedure, we after optimizing over the vectors $\uu_r$, tune the weights $c_r$ from the sub-problem
\begin{equation}\label{eq:c_probl}
\min_\mathbf{c} \mathcal{L} ( k(\mathbf{X,U})\mathbf{c} , \mathbf{y}) +\lambda\|\mathbf{c}\|_a  
\end{equation}
in which $a$ equals either 1 or 2; in the former case with sparse regularisation we can think of filtering out unnecessary $\uu_r$ vectors (if initial $R$ is estimated higher than needed). 
Here the optimisation is over $\mathbf{c}$ of length $R\ll n$ and thus not as computationally significant as optimising the $\mathbf{u}_r$.

Due to the non-convex optimisation problem, the optimisation procedure needs to be repeated several times using different initialisation for the $\uu_r$ (by default $K=5$ iterations are used). The two-step procedure for this sparse pre-image kernel machine (\PKM{}) is summarized in Algorithm~\ref{algo:preimg}, which can be implemented for either classification or regression, by choosing an appropriate loss function.

\begin{algorithm}[tb]\caption{Sparse pre-image kernel machine (\PKM{})}\label{algo:preimg}
\begin{algorithmic}
\STATE {\bfseries Input:}  $\mathbf{X}\in\R^{n\times d}$; \, $\mathbf{y}\in\{-1,+1\}^n$;  \, $R\in\mathbb{N} \, ( R \ll n)$; \, 
a kernel function $k$ and its gradient; \, $\gamma$ to control sparsity in $\mathbf{U}$; \, $\lambda>0$; \, $K$

\STATE Initialise $\mathbf{c}\subset \mathbf{y}$ 

\FOR{$i=1$ {\bfseries to} $K$}
\STATE Initialise $\mathbf{U}_i\in\R^{R,d}$ randomly
\REPEAT
  
  \STATE // {\em step 1: update $\mathbf{U}$ with gradient-based method}
  \STATE $\mathbf{U}_i\gets \argmin_{\mathbf{U}} \mathcal{L}\left( k(\mathbf{X}, \mathbf{U})\mathbf{c}, \mathbf{y} \right) + \gamma \Theta(\mathbf{U})$ 
  
  \STATE // {\em step 2: update $\mathbf{c}$}
  \STATE{$\mathbf{c}\gets \argmin_\mathbf{c} \mathcal{L}  (  k(\mathbf{X,U})\mathbf{c} , \mathbf{y}) + \lambda \|\mathbf{c}\|_a  $}
\UNTIL{Convergence or maximum number of iterations}

\ENDFOR

\STATE {\bfseries Output:} $\mathbf{U}$, $\mathbf{c}$ % 
\end{algorithmic}
\end{algorithm}

% ---------------------------------------------------------------------------------------------

\subsection{SPKM extensions}

It is easy to extend our framework to combinations of different kernels. Namely, we can consider
\begin{equation}\label{eq:model}
  f(\cdot)=\sum_{j=1}^{J}\sum_{r=1}^{R_j} c_{rj} \phi_j(\uu_{rj})  ,
\end{equation}
in which the $J$ kernels may or may not work with same set of data. 
It is then natural to consider a dataset where the data can represented with several views, $x_i=(\x_i^1, ..., \x_i^V)$.
The views can be of vastly different data types, but they all describe the same sample and correspond to a common label. 
In multi-view learning one of the most prominent approaches is the multiple kernel learning~(MKL)~\citep{gonen2011multiple}, in which a combination of kernels from different views are used for the learning task, and these combination weights are learned.
Our framework thus extends also to the MKL paradigm, but instead of having one, general multiplier for the whole view, we find and weight basis vectors of the views, while at the same time controlling both primal and dual sparsity of the latent factors.  
With the MKL extension our SPKM can thus be seen as a kernel learning method. Yet, even with only one kernel this can be argued to be the case: learning the optimal $\uu_r$ within a given learning problem can be seen as learning a part of the kernel, since we learn the optimal elements with which the kernel should be used with and in this sense the kernel is not fixed in advance.

We can also consider the multi-class classification problem, and formulate a version of SPKM which learns basis vectors for the different classes. Namely, we can consider matching $k(\mathbf{X}, \mathbf{U})\mathbf{C}$ to one-hot encoded labels $\mathbf{Y}\in n\times c$ (containing values -1 and +1). Here the $\mathbf{U}\in R\times d$ is shared between the tasks, while $\mathbf{C}\in R \times c$ weights these for the various binary classifications indicated by the columns of $\mathbf{Y}$. The approach is similar to one-vs-all scheme, however in this extension all the classifiers are learned jointly via the shared $\mathbf{U}$. 

It is also possible to extend our method to deal with data coming in pairs, as $(\mathbf{x}_i,\mathbf{z}_i)\in \mathcal{X}\times \mathcal{Z}$. For pairwise data, a natural way to construct a kernel is to consider $k((\mathbf{x},\mathbf{z}),(\mathbf{x}',\mathbf{z}'))=k_x(\mathbf{x},\mathbf{x}')k_z(\mathbf{z},\mathbf{z}')$, in which case $\phi((\mathbf{x},\mathbf{z}))=\phi(\mathbf{x})\otimes\phi(\mathbf{z})$. In solving the pre-image problem with pairwise kernels, we consider pairs of pre-images $(\mathbf{u}_r, \mathbf{v}_r)\in  \mathcal{X}\times \mathcal{Z}$. In this case, the algorithm is modified slightly: instead of a single step for solving the $\mathbf{u}_i$, we now iterate over solving $\mathbf{u}_i$ and $\mathbf{v}_i$.
Similarly to regular features, it is easy to consider also MKL-like combinations of different views and/or kernels also with the pairwise kernels \citep{cichonska2018learning}.

% ---------------------------------------------------------------------------------------------
\subsection{Rademacher complexity bound}\label{sec:theory}

\input{theory}

% ---------------------------------------------------------------------------------------------

\section{Illustrations}\label{sec:illustration}

In this section we illustrate our \PKM{} method, first in simple classification and regression settings, then with the popular MNIST dataset for multi-class classification.

\begin{figure*}[tb]
\centering
\begin{subfigure}{\textwidth}
  \centering
    \includegraphics[width=0.24\linewidth]{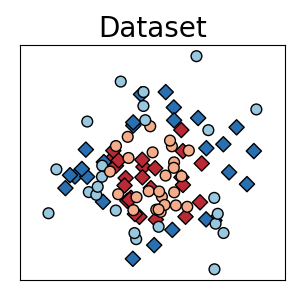}
    \includegraphics[width=0.24\linewidth]{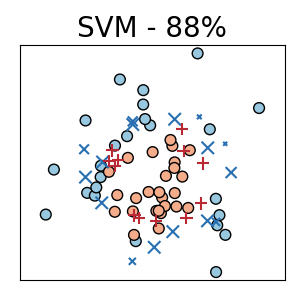}
    \includegraphics[width=0.24\linewidth]{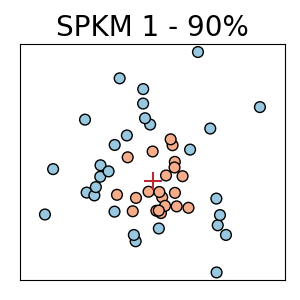}
    \includegraphics[width=0.24\linewidth]{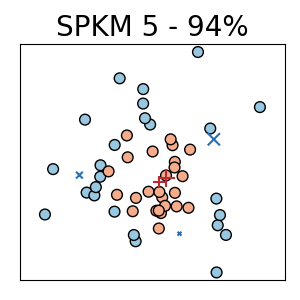}
    \caption{Classification -- performance is measured with accuracy score.}
    \label{fig:illustration_cl}
\end{subfigure}%

\begin{subfigure}{\textwidth}
  \centering
    \includegraphics[width=0.24\linewidth]{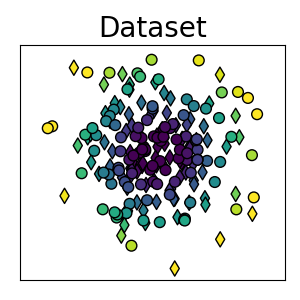}
    \includegraphics[width=0.24\linewidth]{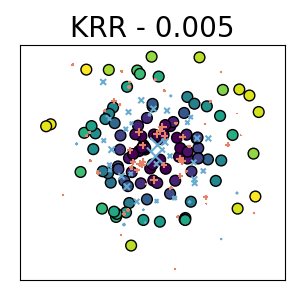}
    \includegraphics[width=0.24\linewidth]{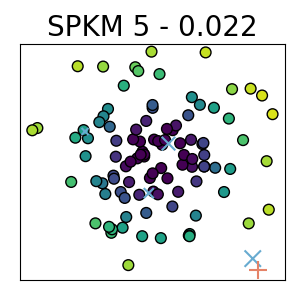}
    \includegraphics[width=0.24\linewidth]{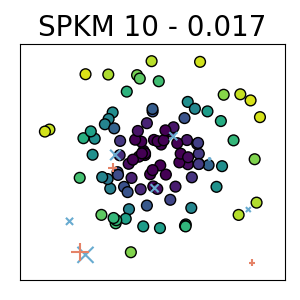}
    \caption{Regression -- performance is measured with mean squared error.}
    \label{fig:illustration_regr}
\end{subfigure}
\caption{
Classification (a) and regression (b) illustrations of \PKM{}. The datasets are shown in the leftmost images; training samples with diamonds and test samples with circles. The colour corresponds to the labels. Predictive models are identified in the figure titles, along with the performance measure, and for SPKM models we further display there the $R$ parameter. Basis vectors/support vectors are shown with crosses, whose size correlates with accompanying weights (values are not comparable between methods), "x" for negative, or "+" for positive vectors. All the methods within an experiment use the same RBF kernel.}
\label{fig:illustration}
\end{figure*}

Figure~\ref{fig:illustration} illustrates our \PKM{} method in classification and regression settings. 
In the classification setting (Figure~\ref{fig:illustration_cl}), the SVM algorithm determines 28 of the 50 training samples to be support vectors and obtains accuracy of 88\% of the test samples. 
Meanwhile, applying \PKM{} method with the cosine proximity loss and searching for only one suitable basis vector ("SPKM 1" in the title), we find an answer obtaining 90\% accuracy, with just one pass of the two-step procedure. 
Adding more basis vectors to \PKM{} algorithm gives also very good results with respect to accuracy.
It is easy to see from the illustration, that unlike SVM which uses the support vectors as boundary descriptions, our method can instead focus on finding a set of basis vectors that describe the data from the two classes as well as possible.
The regression setting is shown similarly in Figure~\ref{fig:illustration_regr}. While KRR performs the best, our \PKM{} can reach similar performance when we add suitably many basis vectors.

\begin{figure}[tb]
    \centering
    \includegraphics[width=0.13\linewidth]{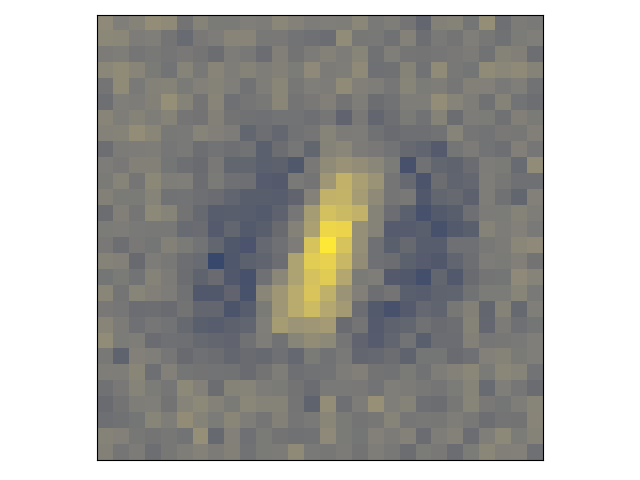}
    \includegraphics[width=0.13\linewidth]{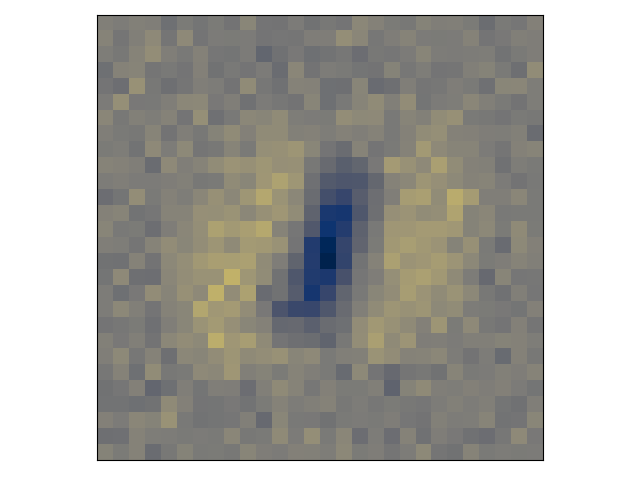}
    \includegraphics[width=0.13\linewidth]{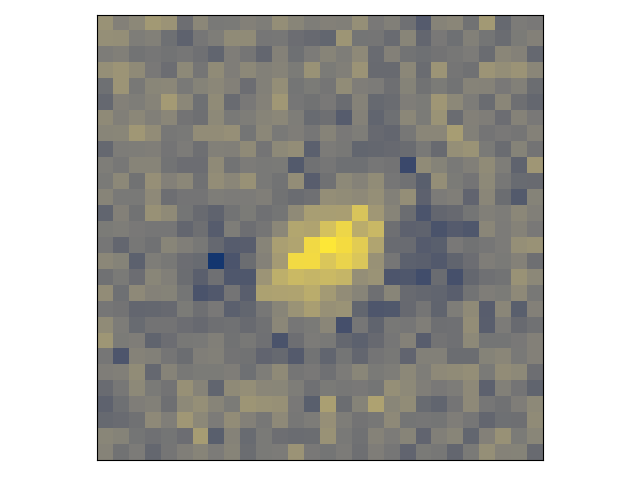}
    \includegraphics[width=0.13\linewidth]{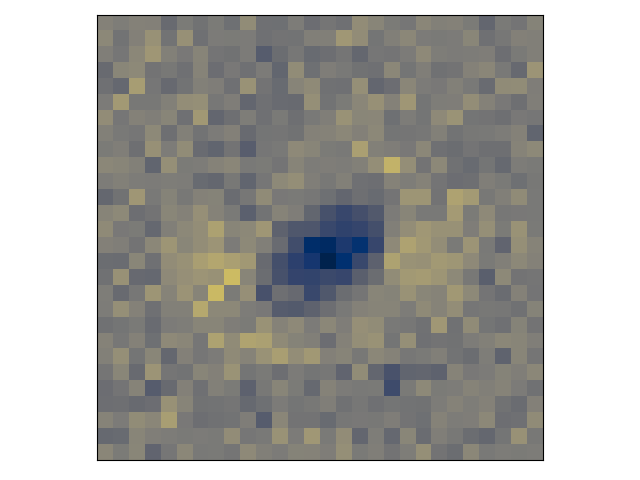}
    \includegraphics[width=0.13\linewidth]{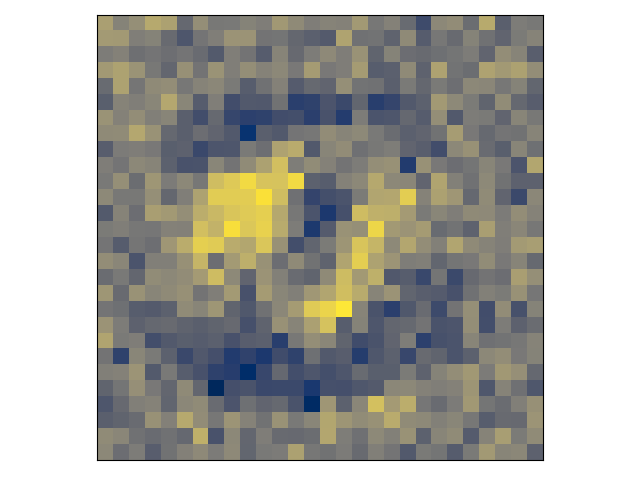}
    \includegraphics[width=0.13\linewidth]{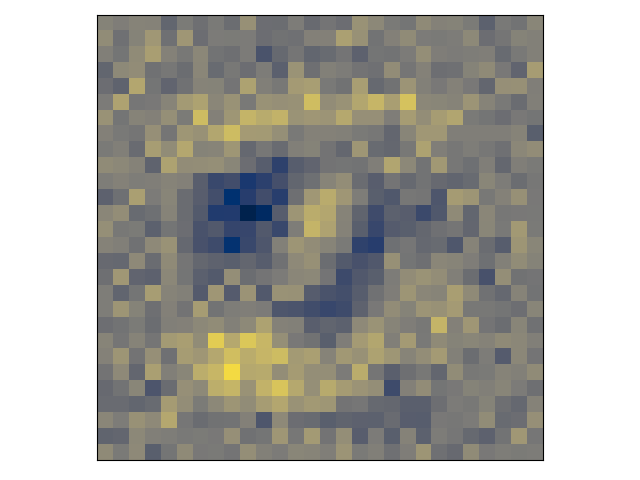}
    
    \caption{Some basis vectors found by \PKM{} applied to MNIST with OVO strategy. Leftmost two: 0 vs 1, middle two: 0 vs 8, rightmost two: 3 vs 9. }  
    \label{fig:mnist_examples}

\medskip 

    \centering
    \includegraphics[width=0.09\linewidth]{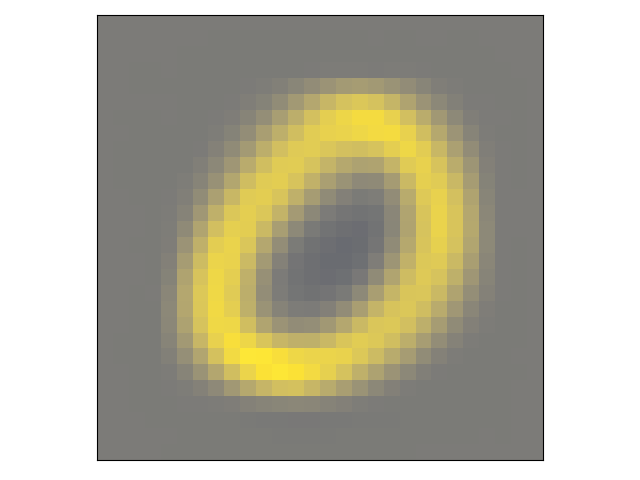}
    \includegraphics[width=0.09\linewidth]{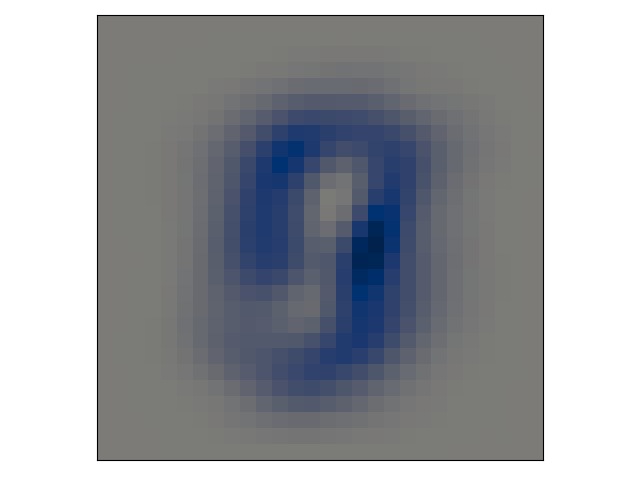}
    \includegraphics[width=0.09\linewidth]{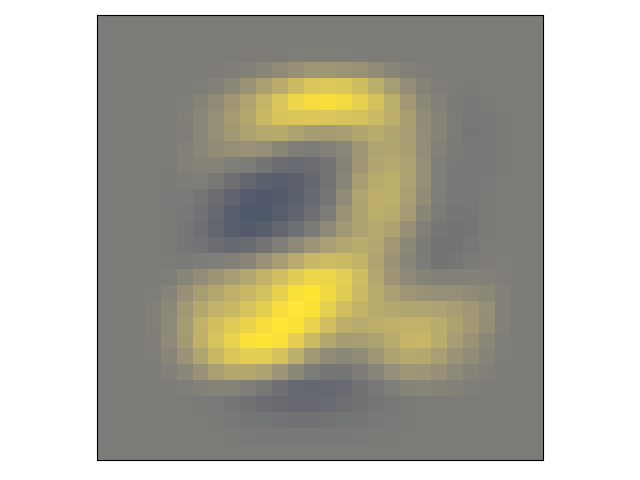}
    \includegraphics[width=0.09\linewidth]{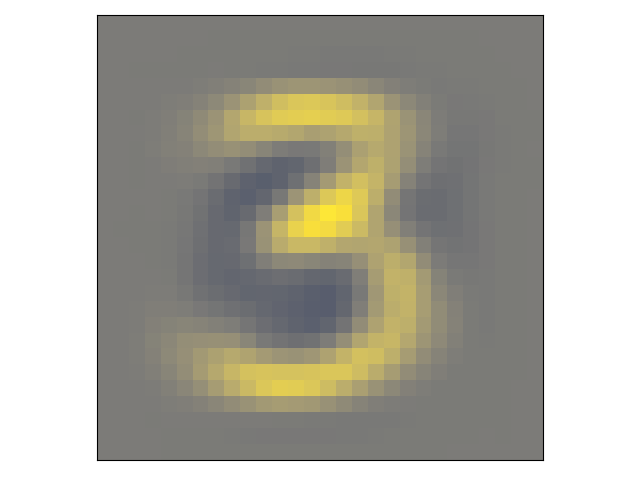}
    \includegraphics[width=0.09\linewidth]{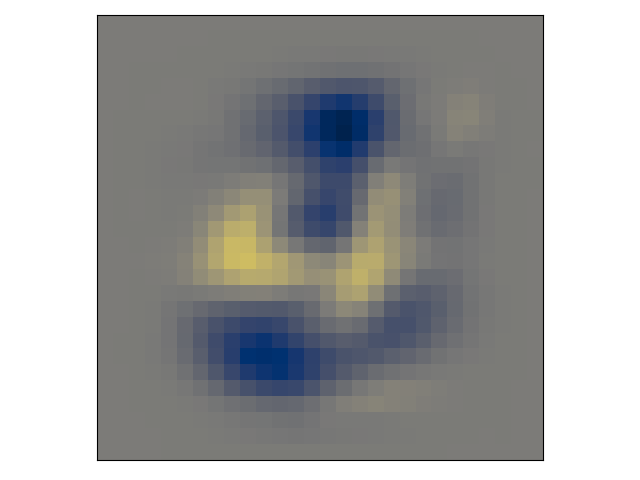}
    \includegraphics[width=0.09\linewidth]{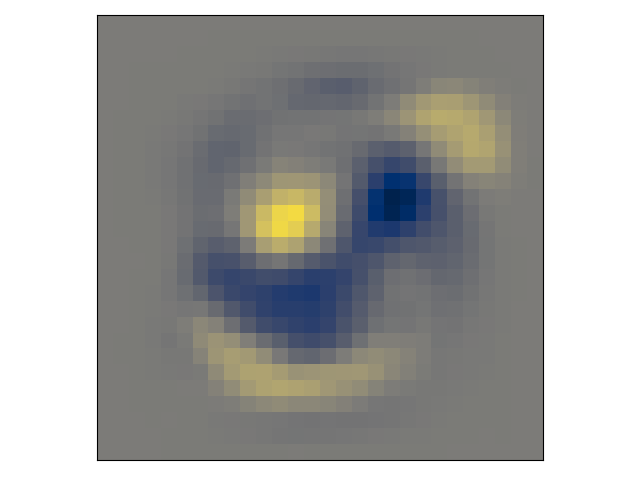}
    \includegraphics[width=0.09\linewidth]{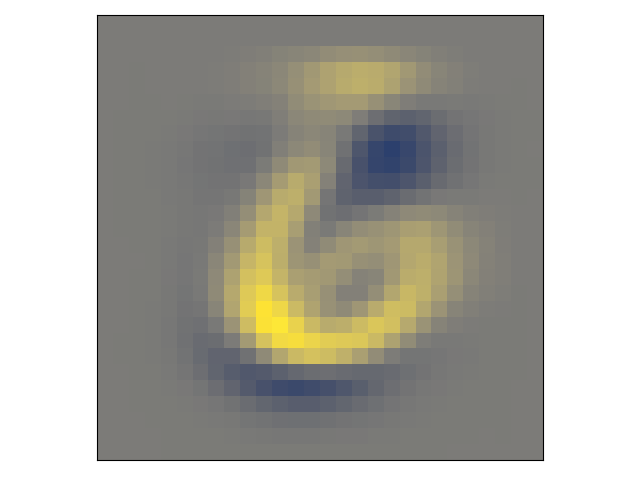}
    \includegraphics[width=0.09\linewidth]{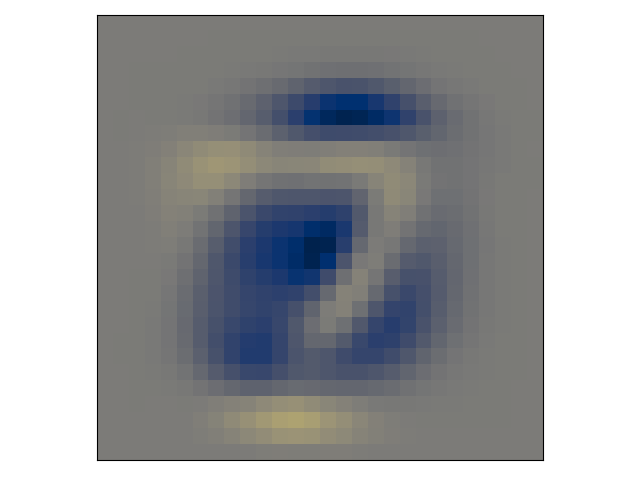}
    \includegraphics[width=0.09\linewidth]{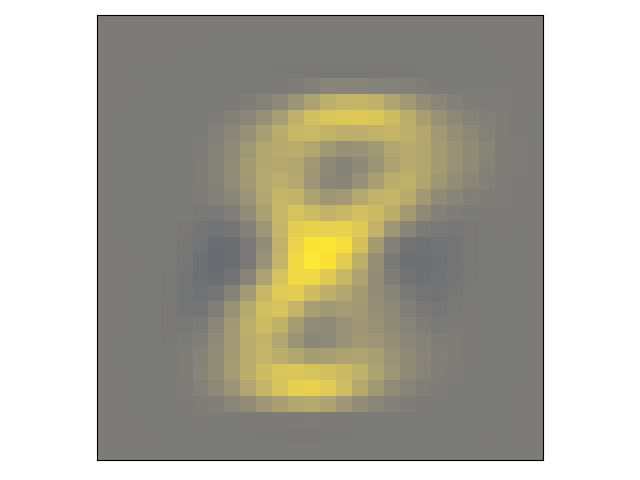}
    \includegraphics[width=0.09\linewidth]{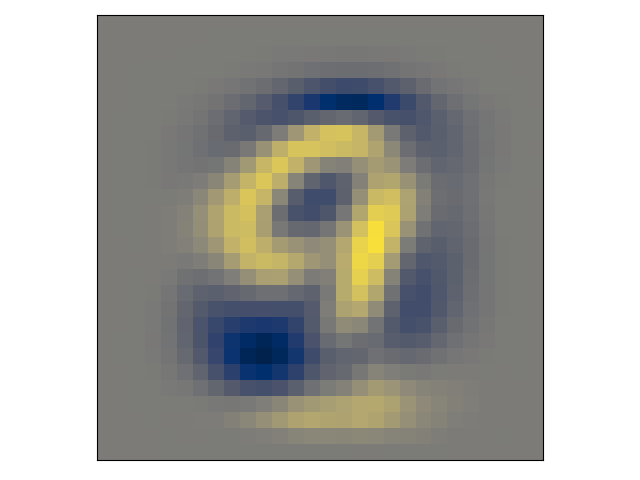}
    \caption{Basis vectors found by the multi-class \PKM{} algorithm applied to MNIST with ten basis vectors.}  
    \label{fig:mnist_mc_examples}

\end{figure}

We further illustrate our \PKM{} and the basis vectors it finds with the MNIST dataset.\footnote{\url{http://yann.lecun.com/exdb/mnist/} .} The dataset consists of handwritten digits, and has become a benchmark in many machine learning approaches.
We consider 6000 training and validation samples, and test on the separate 10000 samples. On this dataset, SVM using one-versus-one multiclass strategy, while obtaining accuracy of about 95\%, uses over 2000 support vectors to describe the solution.
We apply our \PKM{} algorithm with OVO and native multiclass strategies. In the first case, we search for only two basis vectors in each binary classifier (90 in total) using the cosine proximity loss, resulting in around 90\% accuracy, despite the extreme dual sparsity. Some examples of the obtained basis are displayed in Figure~\ref{fig:mnist_examples} (all of them are shown in supplementary material), showing distinguishing features between the two classes of the binary models. 
The basis vectors found by of \PKM{} using the native multiclass strategy with squared loss when limiting to 10 basis vectors are displayed in Figure~\ref{fig:mnist_mc_examples}. Despite the very low number of basis vectors, the model obtains around 80\% accuracy, and it is easy to interpret the basis vectors as prototypes for the ten classes.

%% file: pre-img_img.tex
\begin{figure}
    \centering
    \begin{tikzpicture}[scale=0.85]
%% RKHS Blob
\node[] at (0.9,2.2) {$\mathcal{H}$};
\path[draw,use Hobby shortcut,closed=true]
(0,0) .. (.4,0.8) .. (0.6,2) .. (.3,2.3) .. (-1.1,1.5) .. (-1,.5);
% \node[] at (0,2) {$w'$};
% \node at (0, 1.7) {\textbullet};

%% Im(X) blob
\node[] at (1.4,0.8) {Im($\mathcal{X}$)};
\path[draw,use Hobby shortcut,closed=true]
(-.1,0.4) .. (.2,1) .. (-.5,1.3) .. (-0.75,0.8) .. (-0.4,.3);
\node[] at (-0.3,0.8) {$\phi(\uu)$};
\node at (-.3, 0.5) {\textbullet};
\path[draw,use Hobby shortcut,closed=false]
(1.2,0.6) .. (1,.5) .. (.1,0.5);

%% Span(S) blob
\node[] at (1.5,1.5) {Span($S$)};
\path[draw,use Hobby shortcut,closed=true]
(-.1,1.1) .. (.3,1.5) .. (-.3,2.1) .. (-0.7,1.8) .. (-0.5,1.1);
\node[] at (-0.2,1.9) {$w$};
\node at (-.2, 1.6) {\textbullet};
\path[draw,use Hobby shortcut,closed=false]
(1.25,1.3) .. (.9,1.25) .. (.25,1.3);

%% X Blob
\node[] at (-4.6,1.7) {$\mathcal{X}$};
\path[draw,use Hobby shortcut,closed=true]
(-3.5,0.2) .. (-4.1,1) .. (-4.3,1.2) .. (-3,1.2);
\node[] at (-3.5,1.5) {$\mathbf{u}$};
\node at (-3.5, 1.2) {\textbullet};

% arrows between blobs
\node (u1) at (-3.5, 1.2+0.05)  {};
\node (u2) at (-3.5, 1.2-0.05)  {};
\node (phiu) at (-.3, 0.5) {};
\node (w) at (-.2, 1.6) {};
\node (qm) at (-1.9,1.9) {};
% \node (w2) at (0, 1.7)  {};
\draw[->] (u1)  to [out=20,in=160] (phiu);
\draw[->] (phiu)  to [out=190,in=-30] (u2);
\draw[->] (qm)  to [out=10,in=160] (w);
\node[] at (-2.7,1.65) {$\phi$};
\node[] at (-1.9,.8) {$\phi^{-1}$};
\node[] at (-2,1.9) {\large ?};

\end{tikzpicture}
    \caption{$\phi(\uu)$ always has a pre-image in data space $\mathcal{X}$ (not necessarily coinciding with the training data $S$), while the representer theorem solution $w$ usually does not.}
    \label{fig:preimg_spaces}
\end{figure}

%% file: theory.tex
We study the generalisation performance of our algorithm via its Rademacher complexity.
We use the standard assumption that a kernel function is bounded, that is, $k(\x, \x)\leq \tau$ for any $\x$. Many kernels fulfill this assumption, including the popular RBF kernel. 
Considering the $l_1$ regularization for the coefficients $c_r$, the hypothesis class for our model is 
\[
    \hat{\mathcal{P}}=\left\{ \x\mapsto \langle w, \phi(\x) \rangle : w=\sum_{i=1}^R c_i\phi(\uu_i), \|c\|_{1}<\Lambda, \uu_i\in\R^d \right\}.
\]
We have not included here the sparsity regularization over the $\uu_r$, since it turns out it is not needed for obtaining a bound. The constraint that $\uu_i\in\R^d$ is specific to our framework as so far as we need vectorial data in order to compute derivatives; however the bound remains the same even with $\uu_i\in\mathcal{X}$ with more general $\mathcal{X}$.
The following theorem (the proof can be found in the supplementary material) provides a Rademacher complexity bound for the hypothesis class $\hat{\mathcal{P}}$.
\begin{theorem}
The Rademacher complexity for our hypothesis class $\hat{\mathcal{P}}$ can be bounded as
\begin{equation}\label{eq:rademacher_bound}
    \mathcal{R}_S(\hat{P}) \leq \sqrt{\frac{\tau^2\Lambda^2}{n}}. %= \sqrt{\frac{r^4\Lambda^2}{n}}.
\end{equation}
\end{theorem}

It is important to note that our bound is obtained over both $\mathbf{U}$ and $\mathbf{c}$, even if it turns out that assumptions on $\mathbf{U}$ are not needed if kernel is assumed to be bounded. Even in this case without explicit regularization over $\mathbf{U}$ we have guarantee of generalization. 
This generalisation guarantee is obtained without direcly using the usual assumption of $\|f\|_\mathcal{H}<\Gamma$.  
In our result, compared to standard setting, 
we replace $\Gamma^2$ with $\tau\Lambda^2$. As we regularise the two parts of our $f$ separately ($\mathbf{c}$ and $\mathbf{U}$), we depend on these two parameters, instead of just $\Gamma$.
Moreover, for many popular kernels (such as the RBF kernel we use in our experiments), $\tau=1$, and in this special case our bound equals the traditional. Another point of view to this is to note that in our framework $\|f\|_\mathcal{H}=\mathbf{c}^\top k(\mathbf{U}, \mathbf{U})\mathbf{c}$ and that when the kernel is bounded, bounding $\mathbf{c}$ directly bounds $\|f\|_\mathcal{H}$. 
Yet comparison to the standard setting of learning with kernels is not the most relevant to make. In a sense, we are operating in paradigm comparable to kernel learning, as we learn the optimal samples with which to use the kernel with for the learning task.

A generalisation bound can be obtained from this by using well-known techniques~\cite{bartlett2002rademacher,koltchinskii2002empirical,mohri2018foundations}.

%% file: exp.tex
In this section we experimentally evaluate our proposed Sparse Pre-image Kernel Machine (\PKM{}) algorithm, investigating its performance compared to traditional approaches while focusing on the primal and dual sparsity properties it promotes. 
\footnote{The Python implementation of our proposed \PKM{} can be found in \url{https://github.com/RiikkaHuu/SPKM}.}

We compare our \PKM{} 
to various other methods. 
First of all, we consider kernel ridge regression (KRR) and support vector machines (SVM), 
and in some experiments we consider also Nyström approximation with these methods.
In addition to these, in classification setting we also consider the separable case approximation (SCA) method~\citep{geebelen2012reducing} for reducing the number of support vectors from SVM solution. As a comparison to prototype-based methods, we consider the MMD-critic, a method based on maximum mean discrepancy, introduced in~\citet{kim2016examples} in classification setting. To compare the results with primal sparsity, we consider the linear Lasso model. 

In our experiments we divide the datasets in thirds for training, validation and testing. 
For the competing approaches we cross-validate over the relevant parameters (regularization parameter $\lambda$ for KRR and Lasso, $C$ for SVM and SCA) in range [1e-5,1]. In our proposed \PKM{} method, we cross-validate over the regularization parameter for norm of $\mathbf{c}$, $\lambda$, with these parameters.
For the sparsity inducing regularisation method, we consider the $l_1$-ball projection~\citep{duchi2008efficient}, considering a large range of values (typically [1e-1-1e3]). 
With our \PKM{} we consider five random initialisations for $\mathbf{U}$ and run the algorithm once in classification setting, and five times in regression (maximum number of iterations in Algorithm~\ref{algo:preimg}). Unless otherwise stated, in classification tasks we use the cosine proximity loss, and in regression the squared loss. 
In all the experiments, unless otherwise stated, we use the RBF kernel with parameter $\sigma$ set as the mean of the distances in the training set. As a technical detail, in our SPKM algorithm with RBF kernel we linearly transform the values of the kernel from range $[0, 1]$ to $[-1, 1]$.

% ===========================================

\subsection{Experiments in classification and regression}

\begin{table}
\resizebox{\columnwidth}{!}{%
\begin{tabular}{lc clclclclcl}
\toprule
 &  \multicolumn{2}{c}{BCW} & \multicolumn{2}{c}{Diab.} & \multicolumn{2}{c}{Ionos.} & \multicolumn{2}{c}{QSAR}\\ 
 & acc & b/sv & acc & b/sv & acc & b/sv & acc & b/sv \\
\midrule 

% the contents of the table are created with a python script!
\multirow{3}{*}{SPKM}   & \cellcolor[gray]{0.9}  97.8$\pm$0.7 &  1 & \cellcolor[gray]{0.9}  89.1$\pm$0.9 &  1 & \cellcolor[gray]{0.9}  86.6$\pm$5.0 &  1 & \cellcolor[gray]{0.9} \textit{ 65.6$\pm$2.1} &  1\\
  & \cellcolor[gray]{0.9}  97.8$\pm$0.7 &  2 & \cellcolor[gray]{0.9}  87.4$\pm$2.6 &  2 & \cellcolor[gray]{0.9}  88.6$\pm$1.6 &  2 & \cellcolor[gray]{0.9} \textit{ 65.6$\pm$2.1} &  2\\
  & \cellcolor[gray]{0.9}  98.1$\pm$0.4 &  5 & \cellcolor[gray]{0.9}  87.5$\pm$1.2 &  5 & \cellcolor[gray]{0.9}  88.9$\pm$6.1 &  5 & \cellcolor[gray]{0.9}  83.0$\pm$0.7 &  5\\
\midrule 
KRR   & \cellcolor[gray]{0.9}  98.0$\pm$0.5 & 227 & \cellcolor[gray]{0.9}  94.1$\pm$0.5 & 173 & \cellcolor[gray]{0.9}  89.5$\pm$1.8 & 117 & \cellcolor[gray]{0.9}  85.0$\pm$2.6 & 351\\
\midrule 
SVM   & \cellcolor[gray]{0.9}  98.1$\pm$0.4 & 67 & \cellcolor[gray]{0.9}  91.2$\pm$0.7 & 98 & \cellcolor[gray]{0.9}  90.6$\pm$0.7 & 65 & \cellcolor[gray]{0.9}  83.9$\pm$1.7 & 187\\
\midrule 
SCA  & \cellcolor[gray]{0.9}  98.0$\pm$0.5 & 58 & \cellcolor[gray]{0.9}  89.7$\pm$1.2 & 68 & \cellcolor[gray]{0.9}  89.7$\pm$0.7 & 60 & \cellcolor[gray]{0.9}  79.6$\pm$2.2 & 123\\
\midrule 
\multirow{2}{*}{MMD}  & \cellcolor[gray]{0.9}  93.9$\pm$7.5 &  5 & \cellcolor[gray]{0.9}  77.2$\pm$6.3 &  5 & \cellcolor[gray]{0.9}  78.6$\pm$2.5 &  5 & \cellcolor[gray]{0.9} \textit{ 61.2$\pm$10.7} &  5\\
  & \cellcolor[gray]{0.9}  - & - & \cellcolor[gray]{0.9}  78.7$\pm$5.8 & 10 & \cellcolor[gray]{0.9}  76.9$\pm$5.5 & 10 & \cellcolor[gray]{0.9} \textit{ 64.6$\pm$6.6} & 10\\
\midrule 
Lasso  & \cellcolor[gray]{0.9}  96.9$\pm$1.1 & - & \cellcolor[gray]{0.9}  89.5$\pm$1.1 & - & \cellcolor[gray]{0.9}  83.8$\pm$2.5 & - & \cellcolor[gray]{0.9}  83.8$\pm$1.5 & -\\

\bottomrule
\end{tabular}
}
\caption{Classification results w.r.t accuracy and dual sparsity: "b/sv" stands for the average number of "basis or support vectors" found/used: note that for SPKM and MMD these values need to be precised in advance. Accuracy score in italics indicates performance close to majority vote. MMD for the breast cancer Wisconsin dataset with 10 prototypes resulted in error. 
}\label{tb:basic_cl}
\end{table}

We first investigate SPKM in the simple classification and regression settings, comparing it to relevant baselines. We further consider our method as a prototype method, and investigate the goodness of the prototypes by applying them in Nyström approximation. 

\paragraph{Classification} In classification task, we consider four datasets for binary classification: the breast cancer Wisconsin, early stage diabetes risk prediction, Ionosphere and QSAR biodegradation datasets\footnote{\bcwurl, \url{https://archive.ics.uci.edu/ml/datasets/Early+stage+diabetes+risk+prediction+dataset}, \url{https://archive.ics.uci.edu/ml/datasets/Ionosphere} and \url{https://archive.ics.uci.edu/ml/datasets/QSAR+biodegradation} .}
for which the results have been displayed in Table~\ref{tb:basic_cl}.
We report for our \PKM{} results up to $R=5$. For most of the datasets SPKM 
even with one or two elements works as well as KRR or SVM with many more elements/support vectors. As an exception, the QSAR biodegradation dataset is more difficult, and we need five elements to capture the relevant information required for classification - an amount that is still much lesser than what KRR or SVM require, or what the reduced amount of support vectors the SCA method produced is. Overall the performance even with five basis vectors is very close to the performance obtained by KRR or SVM. The compared prototype method, MMD, in three out of four cases obtains worse performance than our SPKM, even with double the number of prototypes.

As mentioned, we also investigate the suitability of our method as a general way to select meaningful basis elements by using them in Nyström approximation. The Nyström method approximates the kernel matrix by selecting a few landmark elements to which all the data samples are compared to. Often this can be done by randomly sampling the dataset, or using some clustering techniques such as k-means. Here we compare the SVM and KRR performances when they are given a Nyström approxiamted kernel matrices, in which the landmarks for the approximation are chosen either with k-means or our SPKM. We note that this is done mostly in order to show the relevance of our method, since SPKM already learns to classify the data and building an SVM or KRR model on top of it is in that sense redundant.

The results for the various datasets are shown in Figure~\ref{fig:basic_cl_nystrom}, in the same experimental setting as before. It can be seen that using the SPKM-selected basis vectors as landmarks in Nyström method yields higher accuracies than using as many elements found by k-means algorithm whenever the number of the landmarks is very restricted., especially with KRR.

\begin{table}
\small
\begin{center}
\begin{tabular}{lc clclcl}
\toprule
 & \multicolumn{2}{c}{Diab.} & \multicolumn{2}{c}{Ionos.}\\ 
 & acc & b/sv & acc & b/sv \\
\midrule

\multirow{3}{*}{SPKM}   & \cellcolor[gray]{0.9}  85.8$\pm$1.6 &  1 & \cellcolor[gray]{0.9}  76.1$\pm$4.6 &  1\\
  & \cellcolor[gray]{0.9}  87.4$\pm$2.9 &  2 & \cellcolor[gray]{0.9}  78.3$\pm$1.8 &  2\\
  & \cellcolor[gray]{0.9}  87.9$\pm$2.5 &  5 & \cellcolor[gray]{0.9}  86.3$\pm$3.7 &  5\\
\midrule 
KRR   & \cellcolor[gray]{0.9}  92.1$\pm$0.5 & 173 & \cellcolor[gray]{0.9}  83.5$\pm$1.1 & 117\\
\midrule 
SVM   & \cellcolor[gray]{0.9}  95.4$\pm$1.2 & 74 & \cellcolor[gray]{0.9}  88.0$\pm$0.7 & 41\\

\bottomrule
\end{tabular}
\end{center}
%}
\caption{Classification results w.r.t accuracy and dual sparsity when using polynomial kernels of degree 3. "b/sv" stands for the average number of "basis or support vectors" found/used.
}\label{tb:basic_cl_poly3_small}
\end{table}

\paragraph{Classification - polynomial kernel} In most of our experiments we focus in the RBF kernels, since they are extremely popular. In order to highlight the versatility of our method we here briefly consider also the polynomial kernel. 

Focusing on the previous classification experiments, we re-run the Diabetes and Ionosphere experiments with polynomial kernel of degree 3. The performance (see Table~\ref{tb:basic_cl_poly3_small}) is overall very similar to the experiments using RBF kernel; with these datasets both of the kernels work equally well. SPKM can be seen to adapt as easily to polynomial kernels as to RBF kernels.

\begin{figure}[tb]
    \centering
    \includegraphics[width=0.38\linewidth]{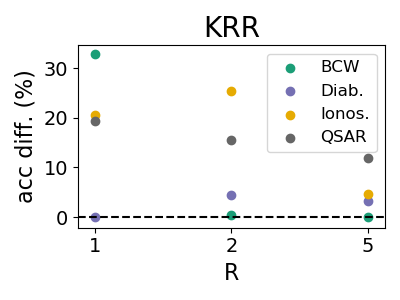}
    \includegraphics[width=0.38\linewidth]{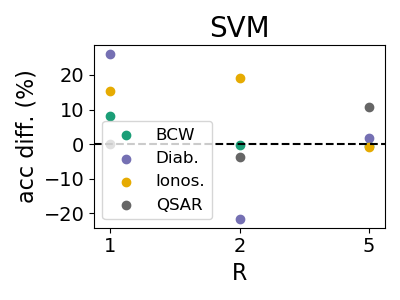} 
    \caption{Comparison on using SPKM and k-means as methods for selecting the landmarks for Nyström approximation with KRR (left) and SVM (right) in the classification experiments. The plots display the differences in average accuracies obtained by the methods: "SPKM-KRR", meaning that values above 0 show superior performance of SPKM-based landmarks.
    }
    \label{fig:basic_cl_nystrom}
\end{figure}

\paragraph{Regression} Similarly, we investigate the performance of our algorithm in the regression setting, with various datasets: the Parkinson's telemonitoring, Boston housing prices, and Sarcos\footnote{\url{https://archive.ics.uci.edu/ml/datasets/Parkinsons+Telemonitoring}, \url{https://archive.ics.uci.edu/ml/machine-learning-databases/housing/}, and \url{http://www.gaussianprocess.org/gpml/data/} .}, of which for the last we consider the first target value. The results are reported in Figure~\ref{fig:basic_regr}, where we show the performance of our proposed \PKM{} as a function of number of basis elements ($R$). Our SPKM converges towards the KRR solution when $R$ increases. Also, using SPKM performs always better than using KRR with same number of landmarks selected by k-means. In two out of three cases also KRR with Nyström approximation obtained with SPKM landmarks performs better than traditional Nyström approximation. In the third case it is possible that normalisation done during SPKM procedure makes the basis vectors not applicable with KRR.

We further investigated the scalability of our approach with the Sarcos dataset. The results can be seen in Figure~\ref{fig:time}. As our algorithm uses more basis vectors, the running times naturally increase. However the scalability of SPKM is much better with respect to the number of samples than KRR, as KRR needs to invert the $n\times n$ kernel matrix.

\begin{figure}
\begin{center}
\includegraphics[width=0.325\linewidth]{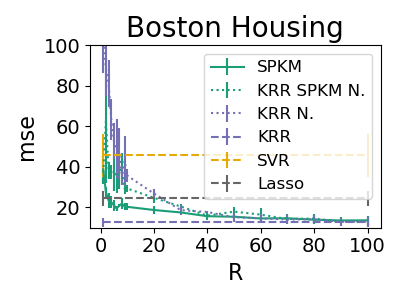}
\includegraphics[width=0.325\linewidth]{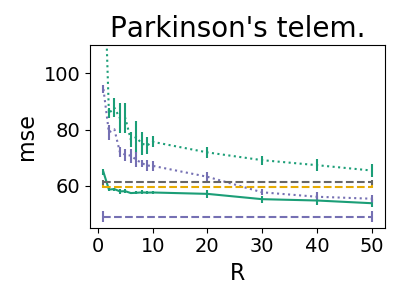}
\includegraphics[width=0.325\linewidth]{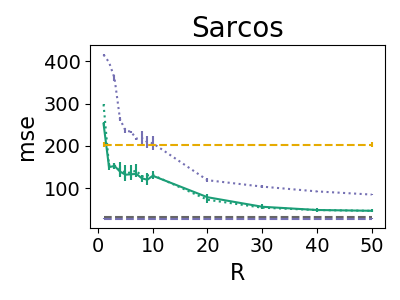}
\end{center}
\caption{Results (mean squared error) on various datasets for regression task, as functions of number of elements in the SPKM algorithm. The compared algorithms (KRR; SVR and Lasso) are plotted as horizontal lines, as they do not depend on the $R$ parameter.   }\label{fig:basic_regr}
\end{figure}

\begin{figure}
    \centering
    \includegraphics[width=0.4\linewidth]{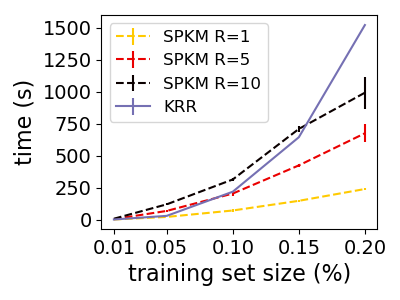}
    \caption{Running time comparison (in seconds) of SPKM and KRR on the Sarcos dataset, with varying amount of data used in training the model. Comparison is also done to the time needed for calculating the full kernel matrix for KRR.
    }
    \label{fig:time}
\end{figure}

\subsection{Primal sparsity}

In this section we investigate the effect of regularizing the primal sparsity in our SPKM algorithm. 
We first consider a simulated dataset: we create low-dimensional data and formulate a binary classification problem with various levels of non-linearity from it. However to the learning algorithms we provide data vectors in which 90\% of the features are just noise, and 10\% make up the original data. Thus the dataset should favour methods that can filter out the unnecessary features. We compare our SPKM (R=2) to Lasso, and report the results in Figure~\ref{fig:cl_sim_sparsity}, alongside with baseline SVM results. The polynomial relations of degree 5 highlight the relevance of our method: it clearly outperforms Lasso in finding relevant features from data, and even outperforms SVM. Notably, it reaches the peak around 80-90\% level of primal sparsity, almost exactly at the amount of meaningful features in data creation~(10\%).

\begin{figure}
\begin{center}
\includegraphics[width=0.32\linewidth]{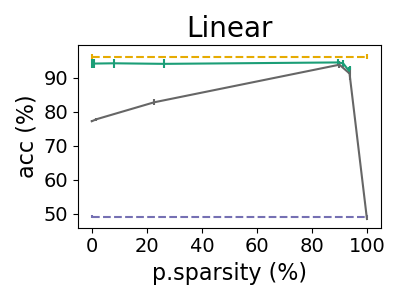}
\includegraphics[width=0.32\linewidth]{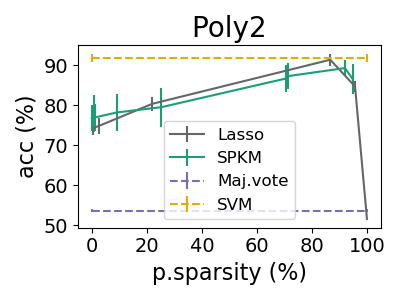}
\includegraphics[width=0.32\linewidth]{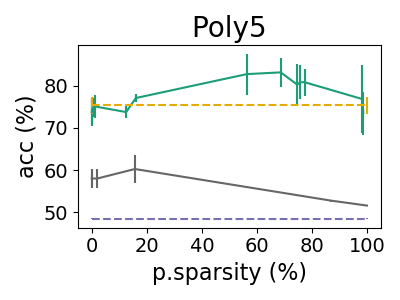}
\end{center}
\caption{Accuracies as functions of primal sparsity for simulated sparse datasets; SVM results do not depend on the primal sparsity and are shown as constants w.r.t sparsity level. The plots display the results on different relations in the data: linear or polynomial of degree 2 or 5. 
}\label{fig:cl_sim_sparsity}
\end{figure}

\begin{figure}
\begin{center}
\includegraphics[width=0.4\linewidth]{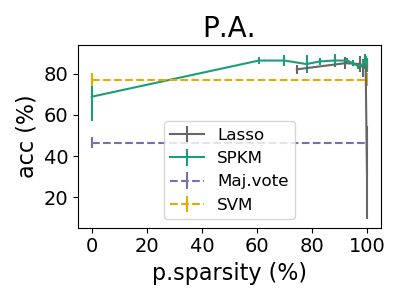}
\includegraphics[width=0.4\linewidth]{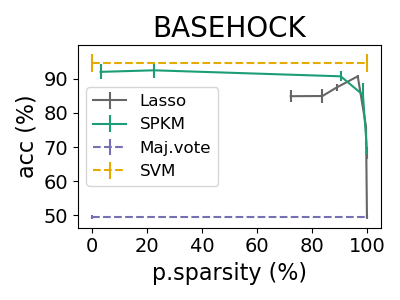} \\
\includegraphics[width=0.4\linewidth]{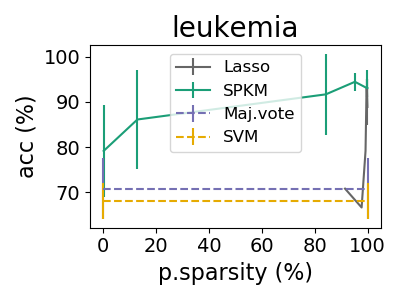} 
\includegraphics[width=0.4\linewidth]{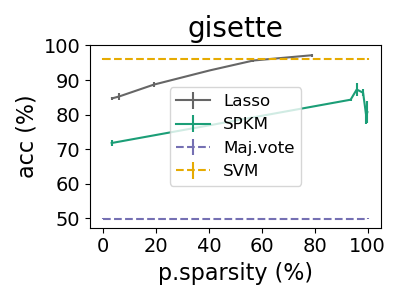} 
\end{center}
\caption{Accuracies as functions of primal sparsity for simulated sparse datasets; SVM results do not depend on the primal sparsity and are shown as constants w.r.t sparsity level.  
}\label{fig:cl_sparsity}
\end{figure}

In addition to the study with simulated data, we consider real datasets: three datasets (BASEHOCK ($n=1993, d=4862$), Leukemia ($n=72, d=7070$) and Gisette ($n=7000, d=5000$)) from the  ASU Feature Selection Repository~\citep{zhao2010advancing}\footnote{\url{https://jundongl.github.io/scikit-feature/datasets.html}},
and a dataset introduced in~\citet{khaledi2020predicting}, on antimicrobial resistance in \textit{Pseudomonas aeruginosa} 
("PA" dataset).\footnote{\paurl}
With this PA dataset we focus on the Tobramycin antibiotic, %\todo{ target 0, Tobramycin}, 
and use the gene expression view (dimensionality: 6026) 
to predict its effect.
We filter out intermediate responses to the antibiotics, being left with 260 data samples in two classes.

The results are presented in Figure~\ref{fig:cl_sparsity} again as accuracies as functions of primal sparsity level. It can be seen that with PA dataset and Basehock our method is competitive with SVM and Lasso, an unlike SVM it finds sparse representations. In the smaller Leukemia dataset Lasso can manage to find a good solution that is extremely sparse, while SPKM finds good solutions for various levels of sparsity, outpefroming also the SVM baseline. On the other hand with the large gisette dataset it is clear that using only two basis vectors in SPKM is not enough to well represent the data.

\subsection{MKL extension}

\begin{table}
\caption{The results of the MKL experiments: accuracies of the various methods for different view combinations. Eeach column represents a view combination, where the top part of the table shows which views are combined by indicating it with a coloured bacground, while "-" marks absence of the view. The numbers in the coloured cells show how many of the two elements the SPKM ends up using (i.e., the corresponding $c_i$ is non-zero) on average from that view. }\label{tb:mvdigits_mkl}
\begin{subfigure}{\textwidth}
\caption{Results on multiple digits dataset.}\label{tb:uwave_mkl}
\begin{center}
\resizebox{\columnwidth}{!}{%
\begin{tabular}{ccccccccccccc}
\toprule
fou& \cellcolor[gray]{0.9} 2& -& -& -& \cellcolor[gray]{0.9} 2& \cellcolor[gray]{0.9} 1\\
zer& -& \cellcolor[gray]{0.9} 1& -& -& \cellcolor[gray]{0.9} 1& \cellcolor[gray]{0.9} 1\\
mor& -& -& \cellcolor[gray]{0.9} 1& -& \cellcolor[gray]{0.9} 1& \cellcolor[gray]{0.9} 1\\
kar& -& -& -& \cellcolor[gray]{0.9} 2& -& \cellcolor[gray]{0.9} 2\\
\midrule
SPKM& 47.00$\pm$1.4 & 47.00$\pm$1.4 & 54.67$\pm$0.9 & 99.00$\pm$0.0 & 92.33$\pm$2.5 & 98.67$\pm$1.2 \\
KRR& 96.33$\pm$0.5 & 98.00$\pm$0.8 & 96.67$\pm$2.4 & 99.00$\pm$0.0 & 99.00$\pm$0.8 & 99.33$\pm$0.5 \\
SVM& 95.67$\pm$0.9 & 96.33$\pm$0.9 & 93.00$\pm$2.4 & 99.67$\pm$0.5 & 97.67$\pm$2.6 & 99.33$\pm$0.5 \\
Lasso& 94.33$\pm$0.9 & 97.67$\pm$0.9 & 98.33$\pm$0.5 & 97.33$\pm$1.2 & 99.00$\pm$0.8 & 98.33$\pm$0.5 \\

\bottomrule
\end{tabular}    
}
\end{center}
\end{subfigure}

\bigskip

\begin{subfigure}{\textwidth}
\caption{Results on uWaveGesture dataset}\label{tb:uwave_mkl}
\begin{center}
\begin{tabular}{cccccccccc}
\toprule
X& \cellcolor[gray]{0.9} 1& -& -& \cellcolor[gray]{0.9} 1\\
Y& -& \cellcolor[gray]{0.9} 2& -& \cellcolor[gray]{0.9} 2\\
Z& -& -& \cellcolor[gray]{0.9} 2& \cellcolor[gray]{0.9} 1\\
\midrule
SPKM& 85.07$\pm$2.1 & 73.52$\pm$7.1 & 77.46$\pm$1.5 & 93.24$\pm$1.4 \\
KRR& 89.58$\pm$3.2 & 86.76$\pm$1.4 & 84.79$\pm$1.6 & 96.62$\pm$1.4 \\
SVM& 88.73$\pm$2.4 & 78.03$\pm$6.2 & 81.97$\pm$2.7 & 95.49$\pm$2.9 \\
Lasso& 80.56$\pm$3.3 & 77.46$\pm$3.8 & 76.06$\pm$5.1 & 90.99$\pm$1.4 \\
\bottomrule
\end{tabular}
\end{center}
\end{subfigure}

\end{table}

To highlight the capabilities of the MKL-extension of our algorithm, we consider the multi-view digits\footnote{\url{https://archive.ics.uci.edu/ml/datasets/Multiple+Features}} and uWaveGesture\footnote{\url{http://www.timeseriesclassification.com/description.php?Dataset=UWaveGestureLibraryAll}} datasets, and train our SPKM algorithm on various combinations of the four and three views, respectively. For the digits dataset we classify "4" and "9", and with the uWaveGesture we consider the classes 2 and 7. 
In all of the views we use the RBF kernels with view-specific $\gamma$ parameter, and for SPKM consider $R_j=2$ using the squared loss and regularizing with $\|c\|_1$. 
In KRR and SVM baselines we consider uniformly weighted combination of the kernels of individual views, while with Lasso we directly concatenate all the data. 

Table~\ref{tb:mvdigits_mkl} displays the results, showing the obtained accuracies with various combinations, as well as how many elements SPKM finds relevant from the views (out maximum $R_j=2$).
We can see that using l1-norm regularization on the multipliers $c_r$, the MKL extension of SPKM can easily indicate the unnecessary views (see e.g. in digits dataset the full combination where the perfect view "kar" is used). Moreover, like traditional MKL, SPKM-MKL is able to combine favourably views that alone are not enough to solve the learning problem, as an example of this is the combination of views "fou", "zer" and "mor" in the digits dataset.

%% file: conclusion.tex
\section{Connections to other works}\label{sec:related}

The idea of learning interpretable models through kernel pre-images has been investigated in~\citet{uurtio2019large} in the context of kernel canonical correlation analysis (KCCA). They focused in the specific task of iteratively finding interpretable CCA projection directions with deflation to maximize the amount of explained cross-covariance. Here, but making use of the kernel pre-images,  we propose a generic supervised learning framework that can assume multiple alternative loss functions and regularizers.

Elsewhere in kernel methods, some researchers have focused on a related task of reducing the number of support vectors in SVM learning. However, in contrast to our work,  many of them~\citep{downs2001exact,geebelen2012reducing,panja2018ms,wang2020condensing} choose the support vectors from the available training samples. An approach closer to our work is given by the so-called reduced set methods~\citep{blachnik2011simplifying,scholkopf2001generalized, burges1996simplified, joachims2009sparse}, where after obtaining the SVM solution, in a post-processing step a new, simpler, hyperplane is searched for such that it still approximates the original solution well. The main point of these methods is, that the new hyperplane might not be characterised with the original data samples, but rather with samples found for example with clustering approaches, or some other scheme such as with eigenvalue decomposition. Alternatively,  one may in a preprocessing step optimise the original data representation with for example clustering approaches, after which ordinary SVM can be trained \citep{neto2013opposite}.  % check the papers they cite
Contrary to the above approaches, in our work we propose to directly find a an optimal set of basis vectors for a given supervised learning task. Additionally, our method obtains both dual and primal sparsity, which gives a better basis for interpreting the models than dual sparsity alone.

As our method finds a few basis elements from which the decision model is built, a natural comparison is the Nyström method~\citep{williams2001using,drineas2005nystrom}. The Nyström method approximates a kernel matrix by selecting a few landmark points, either randomly from the data samples or for example with some clustering method, such as k-means. The full kernel matrix is then approximated so, that the data samples are only compared to the landmark elements instead to all other data samples. This landmark comparison resembles our approach; however there are some major differences. First of all Nyström approximation is obtained without knowledge of the subsequent learning task, contrary to our work where we in a sense learn the kernel jointly with the learning problem. Secondly, the decision function of kernel-based methods does not change even if some kernel approximation scheme, like Nyström method, is applied. This means that even if the kernel approximation only depends on the few samples, the decision function itself still depends on all the training data samples, contrary to our case when the decision function only depends on the found basis elements.

The kernel approximation methods are motivated from the desire to scale up the kernel methods. In this area of research, recently new methods have been proposed~\citep{rudi2017falkon,meanti2020kernel}.
While the complexity of our gradient update $\mathcal{O}(dnR)$ improves on the scaling of the traditional kernel methods ($\mathcal{O}(n^3)$ due to inverting the non-approximated kernel matrix), we wish to point out that the focus of our paper has not been same than these papers. \citep{rudi2017falkon} considers speeding up gradient-based optimisation by using kernel approximations and preconditioning which we have not adopted for simpler and more fair comparisons to the baselines, nor have we optimised our code to work with GPUs like~\citep{meanti2020kernel}. Both of these directions could be applied on top of our method, however such approaches are not the focus of this work.

A connection can also be made to RBF networks, which are neural networks whose activation functions are radial basis functions~\citep{broomhead1988radial}. Considering three layers (input and output with one middle layer with non-linearity) and a Gaussian function for the radial function, the output of the network given an input vector $\mathbf{x}$ is of form \[\psi(\mathbf{x})=\sum_{i=1}^L a_i \exp(-\gamma_i \|\mathbf{x}-\mathbf{c}_i\|^2) , \] in which $L$ is the number of nodes in the hidden layer, $a_i$ are weights to combine the output with and $\mathbf{c}_i$ are so-called centroids associated with the nodes of the inner layer. 
The RBF networks have also been extensively studied theoretically~\citep{krzyzak1998radial,lei2014generalization,niyogi1996relationship}.
While the RBF network model greatly resembles our model, it is only so if we choose to use the RBF kernel function. In our framework it would be equally easy to use for example the polynomial kernel, including the linear kernel as a special case. Moreover with our MKL extension we can handle data coming from several views, with possibly different kernels on them. 
While a backpropagation approach exists for RBF networks, it is common when working with them to just subsample the training data or use some unsupervised learning method in choosing the centroids~\citep{schwenker2001three}. In our SPKM approach the basis vectors $\mathbf{u}_i$ are optimised in a supervised manner.

Finally, our optimisation approach can also be seen to bear some resemblance to matching pursuit approaches~\citep{vincent2002kernel,pati1993orthogonal,davis1994adaptive}, as we are trying to find also the elements themselves in addition to their multipliers. Yet contrary to orthogonal matching pursuit (OMP) we do not consider a dictionary, but the whole input space from which we search our elements. Also, while it could be possible to grow the set of elements in OMP style, we learn all of them at the same time. While OMP has been extended to continuous setting~\citep{keriven2018sketching,keriven2017compressive}, the constraints on the optimisation problem are different from ours.

\section{Conclusion}\label{sec:concl}

In this work we have proposed a novel kernel-based learning method, that enjoys both dual and primal sparsity. The proposed \PKM{} method is scalable with respect to the number of data samples that is usually the bottleneck of kernel methods: the computational complexity of $\mathcal{O}(dnR)$ (with $R\ll n$) instead of the usual $\mathcal{O}(n^3)$. The method is easily adaptable to any differentiable loss function, such as log loss or squared hinge loss. 
Furthermore, our Rademacher bound gives guarantees of the generalisation ability of our \PKM{}.

As with reproducing kernels the input space $\mathcal{X}$ is more arbitrary as our restriction to $\mathbb{R}^d$, in the future it would be interesting to investigate whether a feasible approach for solving the pre-image problem also for non-vectorial data could be devised; perhaps following on the footprints of orthogonal matching pursuit and other dictionary-based approaches.

%% file: appendix.tex
\appendix

% ==========================================================================

 \section{SPKM derivatives}

\subsection{Kernel functions}

\paragraph{RBF kernels} The RBF kernel can be defined as follows:
 \[k_{RBF} (\x,\mathbf{z}) = exp\left( \frac{2\x^\top \mathbf{z} - \mathbf{z}^\top \mathbf{z} - \x^\top \x}{2\sigma^2}\right)\]
Straightforward calculation for derivative gives \[\frac{d}{d \uu} k_{RBF} (\mathbf{x},\mathbf{u}) = \frac{1}{\sigma^2} k(\mathbf{x},\mathbf{u}) \left(\mathbf{u}- \mathbf{x}\right)\]

\subsection{Loss functions}
 
\paragraph{Cosine proximity loss}
The cosine proximity loss with the SPKM model can be written as
\[-\frac{y^\top k(\mathbf{X},\mathbf{U})\mathbf{c}}{\|k(\mathbf{X},\mathbf{U})\mathbf{c}\|},\]
in which $k(\mathbf{X},\mathbf{U})$ is shorthand for the $n\times R$ kernel matrix containing elements $k(\mathbf{x}_i, \mathbf{u}_r)$.

The derivative of this expression can be calculated with the help of quotient rule, and thus the problem reduces in finding the derivatives for numerator and denominator. 

% upstairs = numerator
% downstairs = denominator

\begin{itemize}
    \item Derivative of the numerator can be calculated by first noting that $\mathbf{y}^\top k(\mathbf{X},\mathbf{U})\mathbf{c} = \sum_{i=1}^R \mathbf{y}^\top k(\mathbf{X},\mathbf{u}_i)c_i = \sum_{j=1}^n\sum_{i=1}^R y_jc_ik(\mathbf{x}_j, \mathbf{u}_i).$ We obtain for $u_i$:\[\frac{d}{du_i}\sum_{j=1}^n\sum_{i=1}^R y_jc_ik(\mathbf{x}_j, \mathbf{u}_i) = \sum_{j=1}^n c_iy_j \frac{d}{d \mathbf{u}_i} k(\mathbf{x}_j, \mathbf{u}_i).\]
Here we have $y_j$ multiplying vectors $\frac{d}{d \mathbf{u}_i} k(\mathbf{x}_j, \mathbf{u}_i)c_i$, which are then summed: $\mathbbm{1}^\top diag(\mathbf{y})[D k(\mathbf{X},\mathbf{u}_i)]$, in which $D k(\mathbf{X},\mathbf{u}_i)$ denotes matrix with $\frac{d}{d \mathbf{u}_i} k(\mathbf{x}_j,\mathbf{u}_i)$ on rows.

\item The denominator can be written as $\sqrt{\mathbf{c}^\top k(\mathbf{U},\mathbf{X}) k(\mathbf{X},\mathbf{U})\mathbf{c}}$. The square root can be dealt with the chain rule, so we are left with considering $\mathbf{c}^\top k(\mathbf{U},\mathbf{X}) k(\mathbf{X},\mathbf{U})\mathbf{c}$.
To write this w.r.t $u_i$, let's first consider that $ [k(\mathbf{U},\mathbf{X}) k(\mathbf{X},\mathbf{U})]_{i,j} = \sum_{l} k(\uu_i,\x_l)k(\x_l,\uu_j)$. Now also $\mathbf{c}^\top \mathbf{Ac} = \sum_{i,j}c_i\mathbf{A}_{ij}c_j$, and thus we get $\sum_{l=1}^n\sum_{i,j=1}^R c_i \sum_{l} k(\uu_i,\x_l)k(\x_l,\uu_j) c_j$.

The derivative of that depends on if $i=j$ or not:
\begin{itemize}
\item $i=j$ : $c_ic_j2k(\x_l,\uu_i)\frac{d}{d\uu_i}k(\x_l, \uu_i)$ 
\item otherwise : $c_ic_jk(\x_l,\uu_j)\frac{d}{d\uu_i}k(\x_l, \uu_i)$
\end{itemize}
In the final derivative note that if $i$ does not equal $j$ then the element is in the sum twice; however $i=j$ is there only once. Thus the final formula for the derivative is \[\sum_{l=1}^n\sum_{j=!}^R 2c_ic_jk(\x_l,\uu_j)\frac{d}{d\uu_i}k(\x_l, \uu_i)\]

Writing the sum with matrix multiplications we obtain \[2c_i \mathbf{c}^\top k(\mathbf{U},\mathbf{X}) [Dk(\mathbf{X},\mathbf{U})]\]
in which again $[Dk(\mathbf{X},\mathbf{U})]$ contains the $\frac{d}{d\uu_i}k(\x_j, \uu_i)$ on rows (i.e. for every $\x_j$). 
\end{itemize}

\paragraph{Squared loss}
The loss: \[\|k(\mathbf{X},\mathbf{U})\mathbf{c}-\mathbf{y}\|^2 = \langle k(\mathbf{X},\mathbf{U})\mathbf{c} , k(\mathbf{X},\mathbf{U})\mathbf{c}\rangle + \langle  \mathbf{y}, \mathbf{y}\rangle -2\langle k(\mathbf{X},\mathbf{U})\mathbf{c} , \mathbf{y}\rangle.\]

Writing the norm with inner product, we encounter the terms already investigated with the cosine loss.

% ==========================================================================

\section{Rademacher complexity}

\setcounter{theorem}{0}
\begin{theorem}
The Rademacher complexity for our hypothesis class 
\begin{align*}
   \hat{\mathcal{P}}=\bigg\{ \x\mapsto \langle w, &\phi(\x) \rangle :\\ 
    & w=\sum_{r=1}^R c_r\phi(\uu_r), \|\mathbf{c}\|_{1}<\Lambda, \uu_r\in\R^d \bigg\}. 
\end{align*} assuming $k(\x, \x)\leq \tau$, can be bounded as
\begin{equation*}% \label{eq:rademacher_bound}
    \mathcal{R}_S(\hat{P}) \leq \sqrt{\frac{\tau^2\Lambda^2}{n}}. %= \sqrt{\frac{r^4\Lambda^2}{n}}.
\end{equation*}
\end{theorem}

\begin{proof}

\begin{align*}
    \mathcal{R}_S(\hat{P}) = & \frac{1}{n} \mathbb{E}_\sigma \left[  \sup_{\mathbf{U}, \|\mathbf{c}\|_1 <\Lambda } \left\langle w, \sum_{i=1}^m \sigma_i\phi(x_i) \right\rangle \right] \\
    = & \frac{1}{n} \mathbb{E}_\sigma \left[  \sup_{\mathbf{U}, \|\mathbf{c}\|_1 <\Lambda } \left\langle \sum_{r=1}^Rc_r\phi(u_r), \sum_{i=1}^m \sigma_i\phi(x_i) \right\rangle \right] \\
    \leq & \frac{1}{n} \mathbb{E}_\sigma \left[  \sup_{\mathbf{U}, \|\mathbf{c}\|_1 <\Lambda } \left\|\sum_{r=1}^Rc_r\phi(u_r)\right\|\left\| \sum_{i=1}^m \sigma_i\phi(x_i)\right\| \right]  \\ % \quad \text{Cauchy-Schwarz} \\
    \leq & \frac{1}{n} \mathbb{E}_\sigma \left[  \sup_{\mathbf{U}, \|\mathbf{c}\|_1 <\Lambda } \sum_{r=1}^R|c_r|\left\|\phi(u_r)\right\|\left\| \sum_{i=1}^m \sigma_i\phi(x_i)\right\| \right]  \\ %\quad \text{triangle inequality}\\
    = & \frac{1}{n} \mathbb{E}_\sigma \left[  \sup_{\mathbf{U}, \|\mathbf{c}\|_1 <\Lambda } \sum_{r=1}^R|c_r|\sqrt{k(u_r, u_r)}\left\| \sum_{i=1}^m \sigma_i\phi(x_i)\right\| \right] \\
    \leq & \frac{1}{n} \mathbb{E}_\sigma \left[  \sup_{\|\mathbf{c}\|_1 <\Lambda } \sqrt{\tau}\sum_{r=1}^R|c_r|   \left\| \sum_{i=1}^m \sigma_i\phi(x_i)\right\| \right] \\
    \leq & \frac{\Lambda \sqrt{\tau}}{n} \mathbb{E}_\sigma \left[   \left\| \sum_{i=1}^m \sigma_i\phi(x_i)\right\| \right] \\
    % \leq & \cdots \text{(follow traditional case from here)} \\
    \leq & \frac{\Lambda \sqrt{\tau}}{n} \left[ \mathbb{E}_\sigma   \left[  \left\| \sum_{i=1}^m \sigma_i\phi(x_i)\right\|^2 \right] \right]^{1/2} \\
%\end{align*}
%\begin{align*}
    = & \frac{\Lambda \sqrt{\tau}}{n} \left[ \mathbb{E}_\sigma   \left[ \sum_{i=1}^m  \left\| \phi(x_i)\right\|^2 \right] \right]^{1/2} \\
    = & \frac{\Lambda \sqrt{\tau}}{n} \left[ \mathbb{E}_\sigma   \left[ \sum_{i=1}^m  k(x_i, x_i) \right] \right]^{1/2} \\
    = & \frac{\Lambda  \sqrt{\tau Tr[\mathbf{K}]}}{n} \leq \frac{\Lambda \sqrt{\tau n\tau}}{n}  = \sqrt{\frac{\tau^2\Lambda^2}{n}}\\  % \sqrt{\frac{\tau[\tau\Lambda^2]}{n}} = 
\end{align*}

\end{proof}

A generalisation bound can be obtained from this by using well-known techniques~\cite{bartlett2002rademacher,koltchinskii2002empirical,mohri2018foundations}.

% ==========================================================================

\section{Illustrative experiments: MNIST}

All the basis vectors found by our \PKM{} algorithm with the OVO scheme (applied with 2 basis vectors per binary classifier) are displayed in the Table~\ref{tab:mnist_all_ovo}.

\begin{table*}[]
    \centering
    \begin{tabular}{|c|c|c|c|c|c|c|c|c|c|c|}
\midrule

&0  & 1  & 2  & 3  & 4  & 5  & 6  & 7  & 8  & 9 \\
\midrule
\multirow{2}{*}{0}  & 
 & 
 & 
 & 
 & 
 & 
 & 
 & 
 & 
 & 
\\
 & 
 & 
 & 
 & 
 & 
 & 
 & 
 & 
 & 
 & 
\\
\midrule
\multirow{2}{*}{1}  & 
\includegraphics[width=0.07\linewidth]{mnist_basis_elements/ovo/0-1_u0.png}
 & 
 & 
 & 
 & 
 & 
 & 
 & 
 & 
 & 
\\
 & 
\includegraphics[width=0.07\linewidth]{mnist_basis_elements/ovo/0-1_u1.png}
 & 
 & 
 & 
 & 
 & 
 & 
 & 
 & 
 & 
\\
\midrule
\multirow{2}{*}{2}  & 
\includegraphics[width=0.07\linewidth]{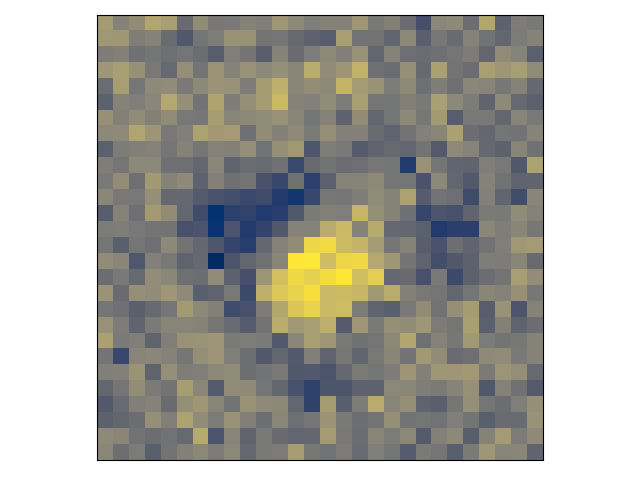}
 & 
\includegraphics[width=0.07\linewidth]{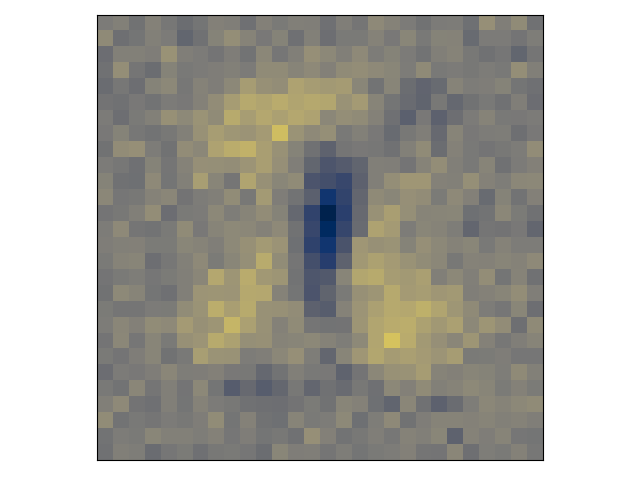}
 & 
 & 
 & 
 & 
 & 
 & 
 & 
 & 
\\
 & 
\includegraphics[width=0.07\linewidth]{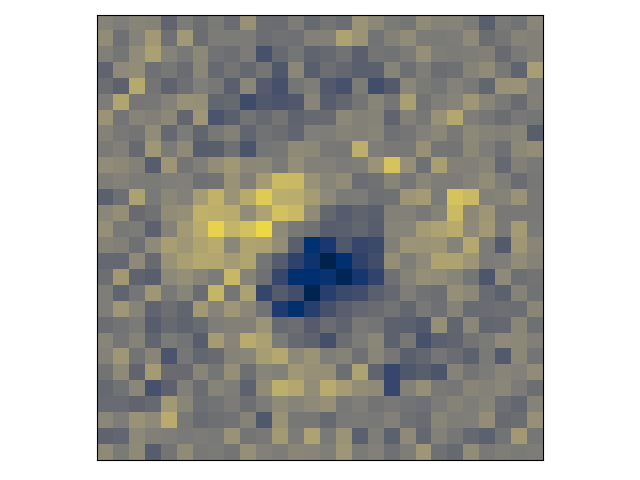}
 & 
\includegraphics[width=0.07\linewidth]{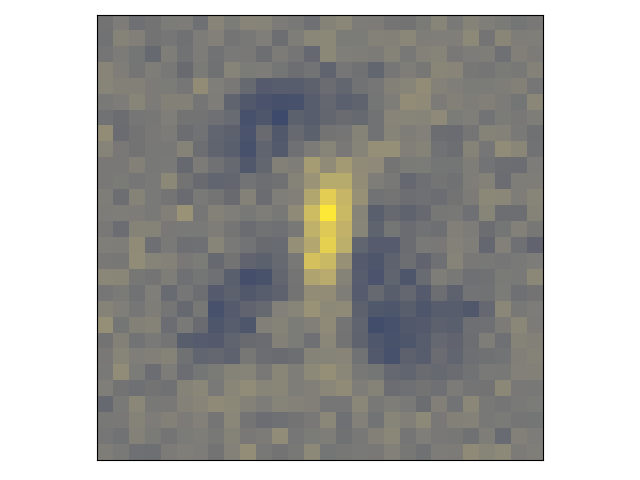}
 & 
 & 
 & 
 & 
 & 
 & 
 & 
 & 
\\
\midrule
\multirow{2}{*}{3}  & 
\includegraphics[width=0.07\linewidth]{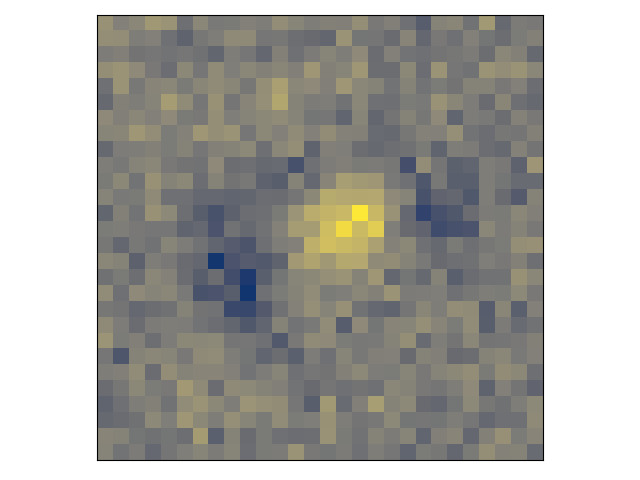}
 & 
\includegraphics[width=0.07\linewidth]{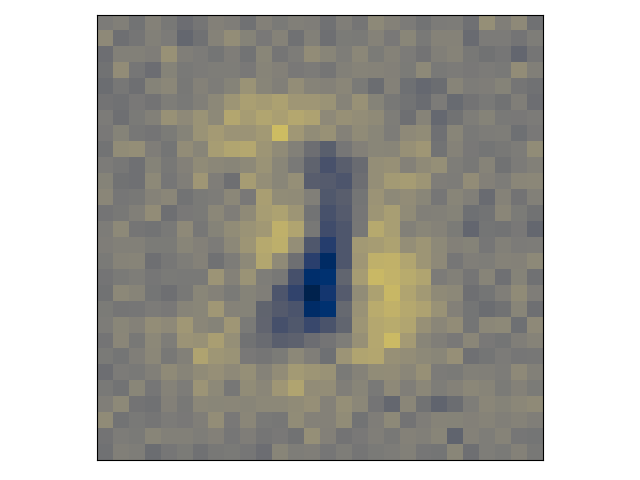}
 & 
\includegraphics[width=0.07\linewidth]{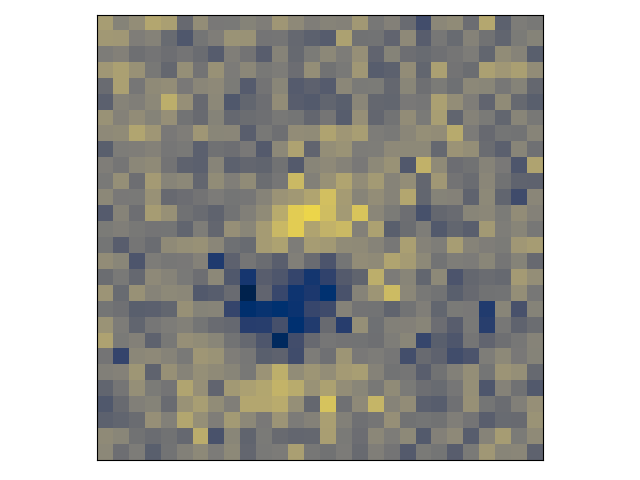}
 & 
 & 
 & 
 & 
 & 
 & 
 & 
\\
 & 
\includegraphics[width=0.07\linewidth]{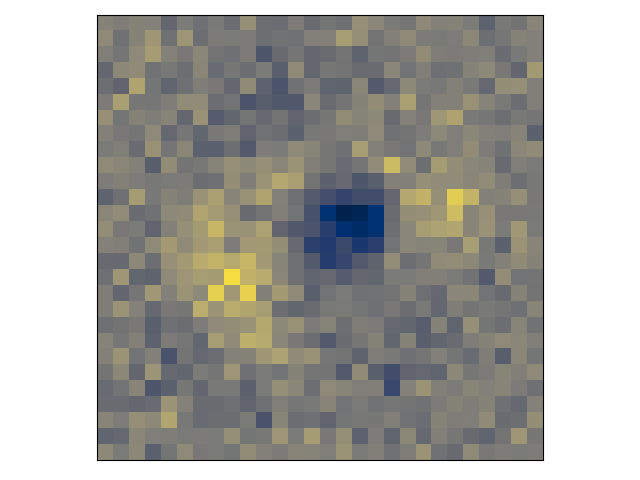}
 & 
\includegraphics[width=0.07\linewidth]{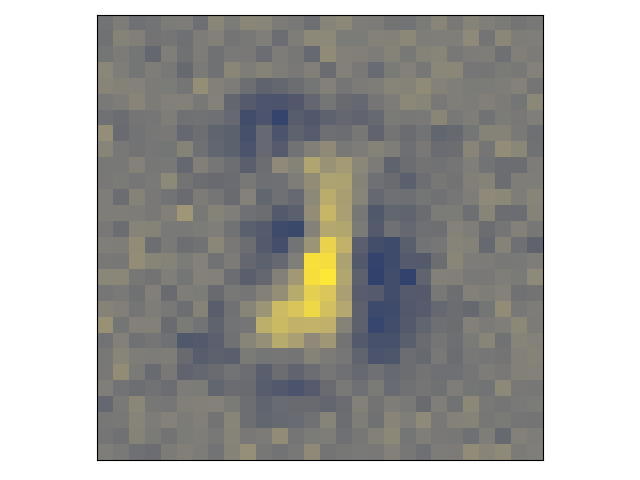}
 & 
\includegraphics[width=0.07\linewidth]{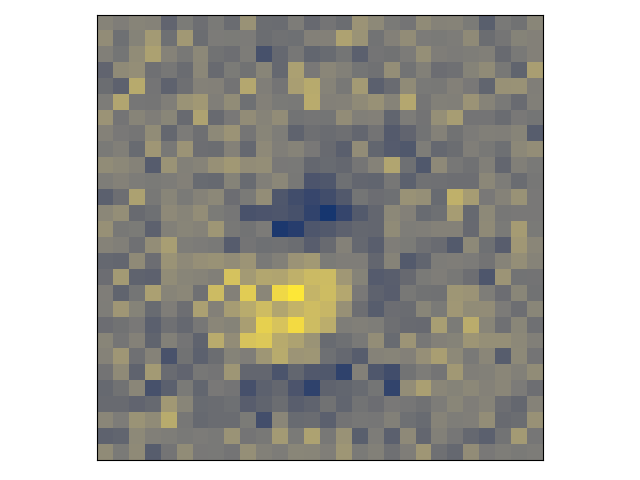}
 & 
 & 
 & 
 & 
 & 
 & 
 & 
\\
\midrule
\multirow{2}{*}{4}  & 
\includegraphics[width=0.07\linewidth]{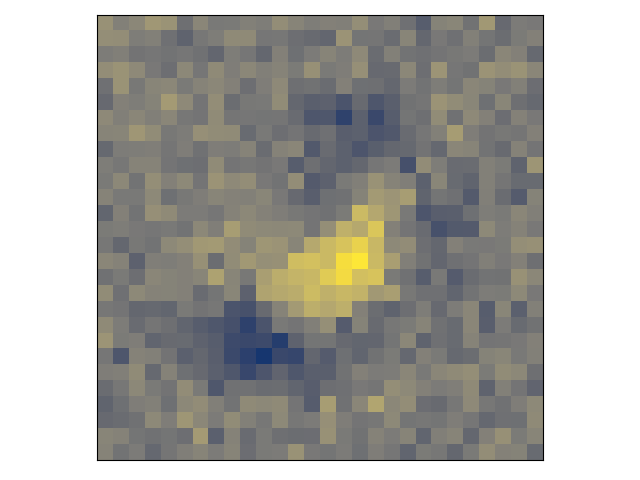}
 & 
\includegraphics[width=0.07\linewidth]{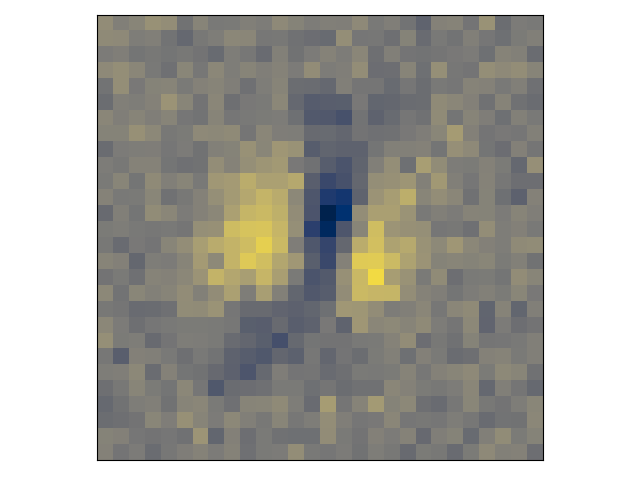}
 & 
\includegraphics[width=0.07\linewidth]{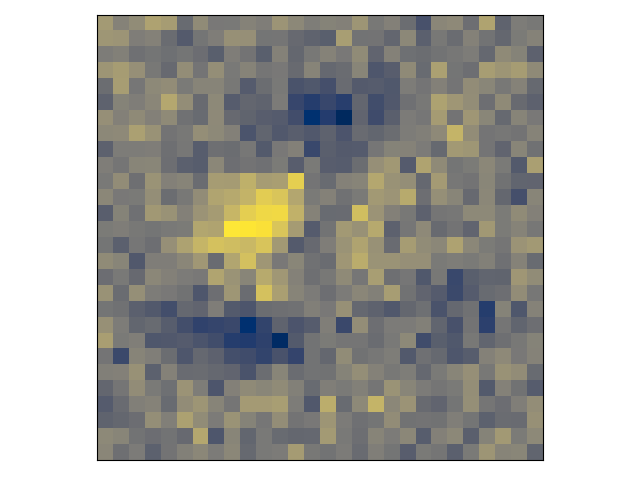}
 & 
\includegraphics[width=0.07\linewidth]{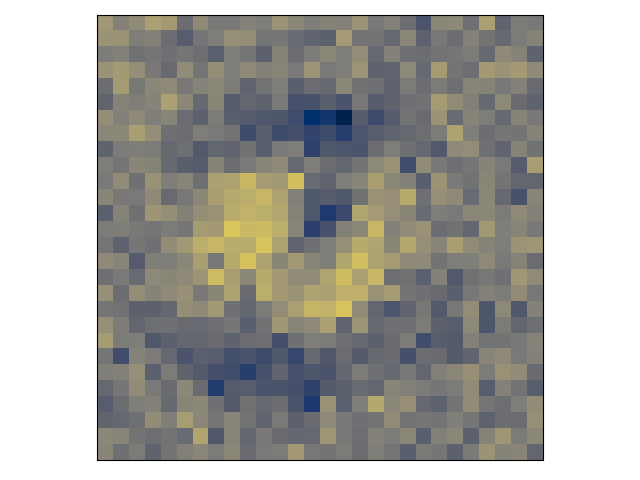}
 & 
 & 
 & 
 & 
 & 
 & 
\\
 & 
\includegraphics[width=0.07\linewidth]{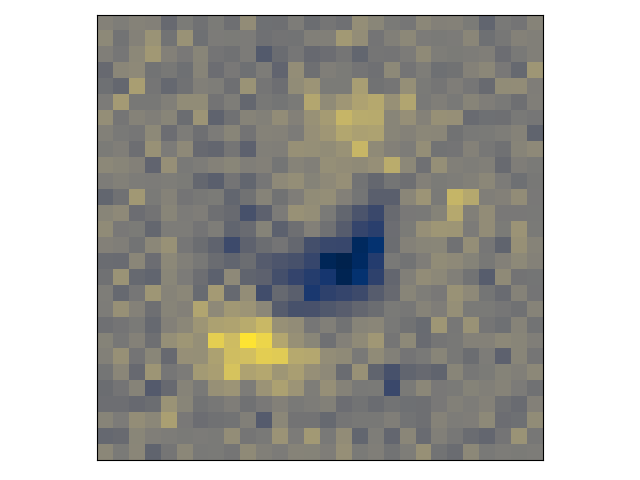}
 & 
\includegraphics[width=0.07\linewidth]{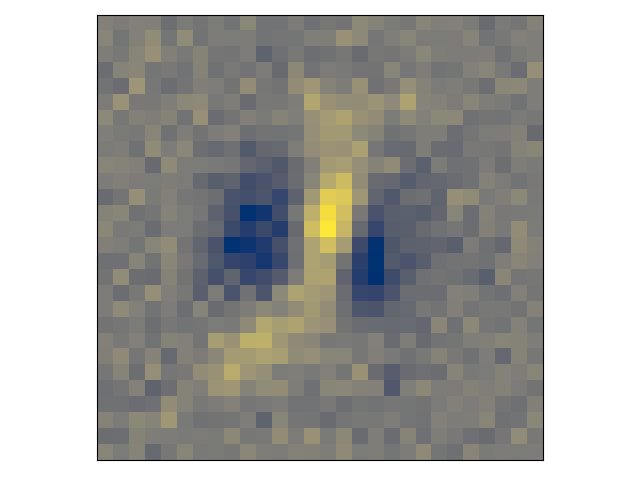}
 & 
\includegraphics[width=0.07\linewidth]{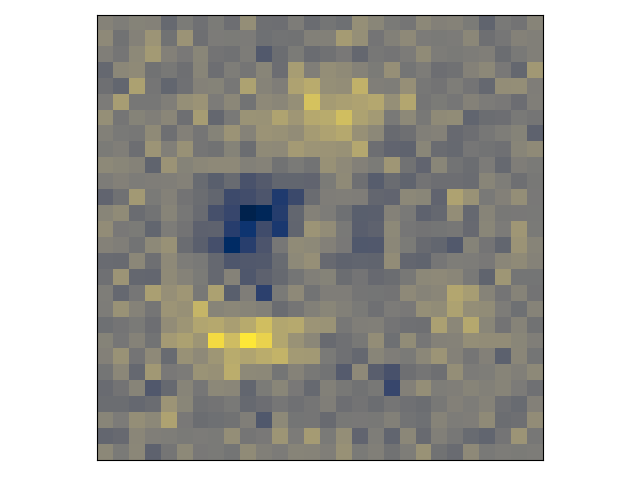}
 & 
\includegraphics[width=0.07\linewidth]{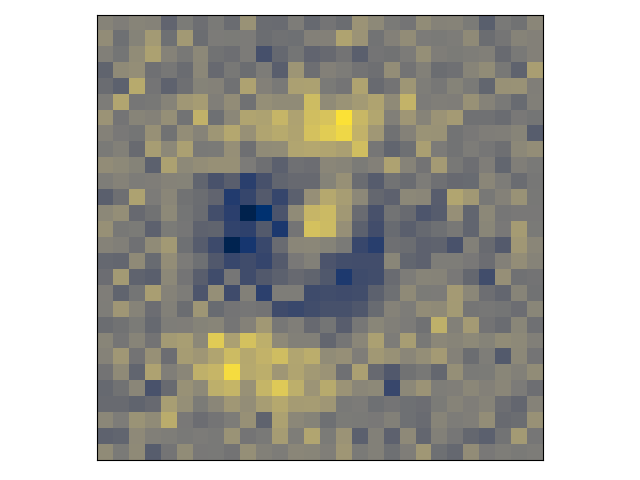}
 & 
 & 
 & 
 & 
 & 
 & 
\\
\midrule
\multirow{2}{*}{5}  & 
\includegraphics[width=0.07\linewidth]{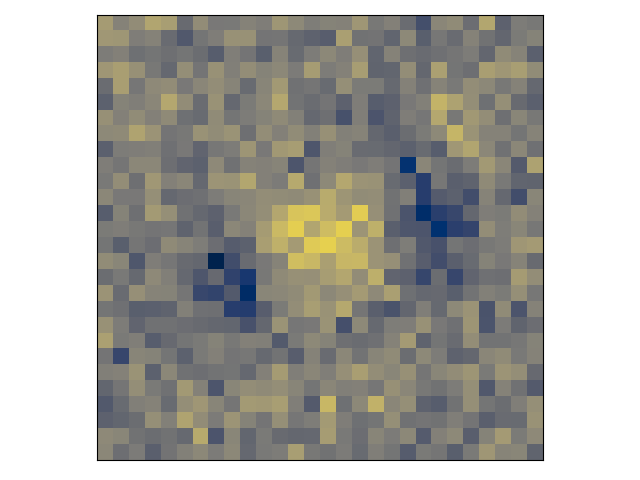}
 & 
\includegraphics[width=0.07\linewidth]{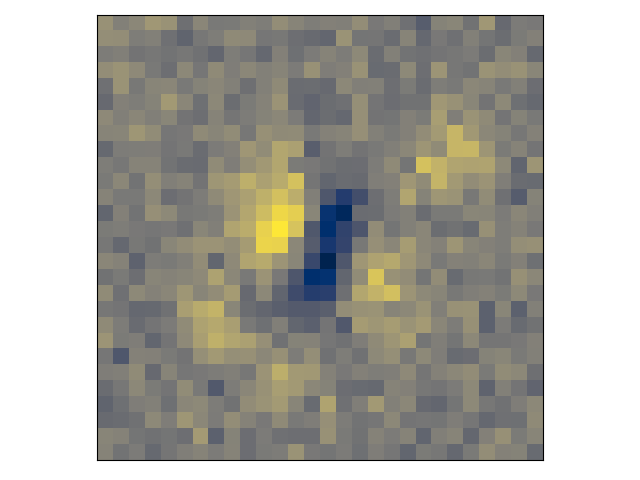}
 & 
\includegraphics[width=0.07\linewidth]{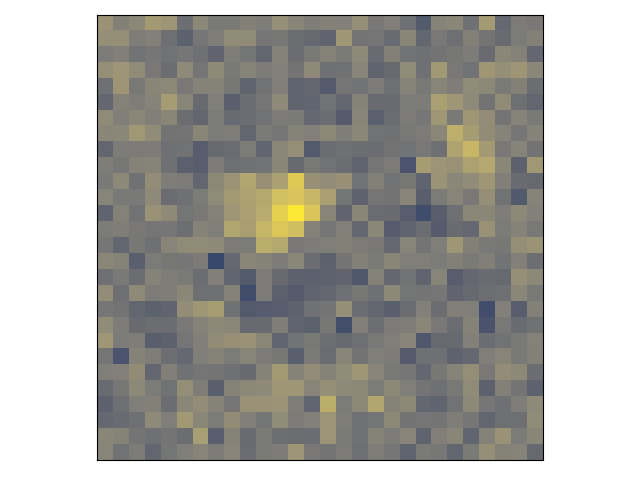}
 & 
\includegraphics[width=0.07\linewidth]{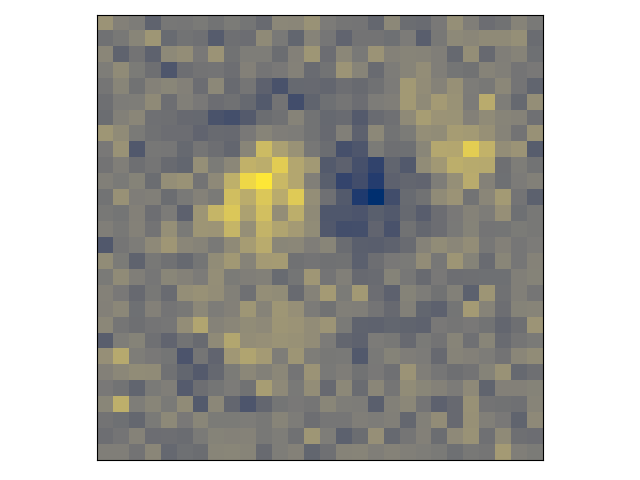}
 & 
\includegraphics[width=0.07\linewidth]{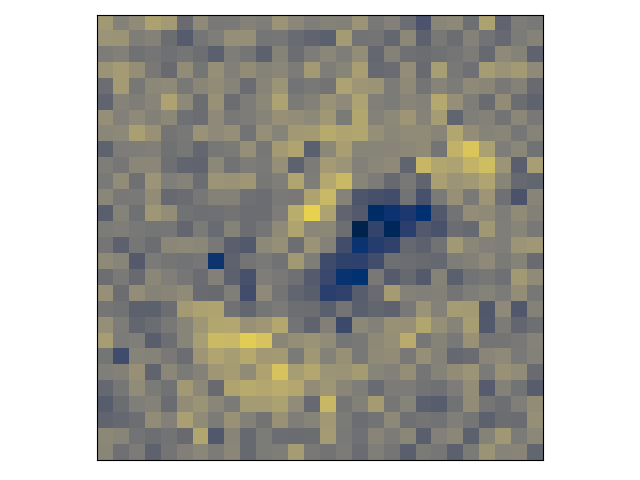}
 & 
 & 
 & 
 & 
 & 
\\
 & 
\includegraphics[width=0.07\linewidth]{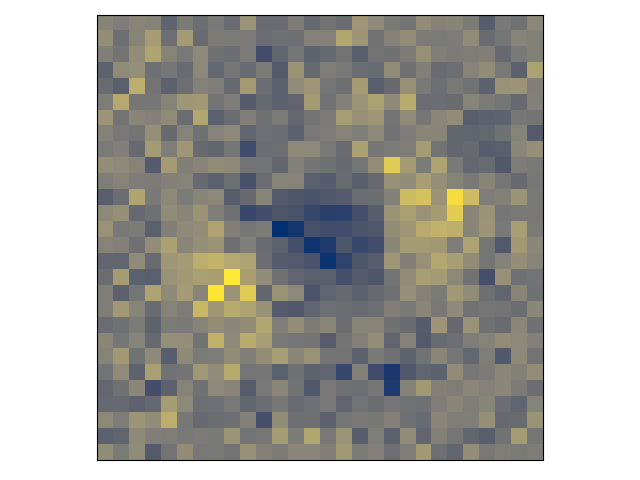}
 & 
\includegraphics[width=0.07\linewidth]{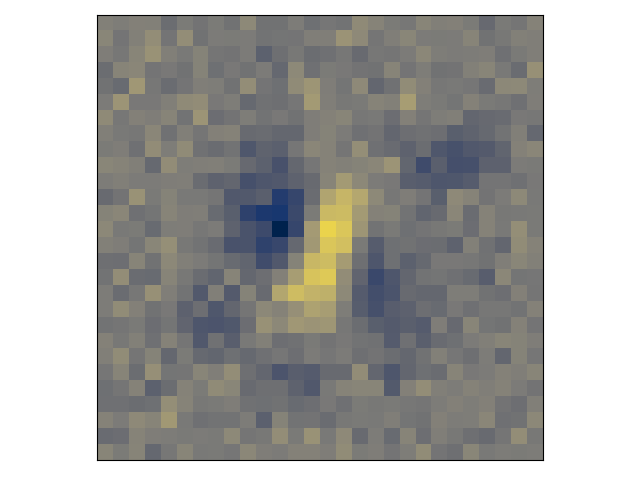}
 & 
\includegraphics[width=0.07\linewidth]{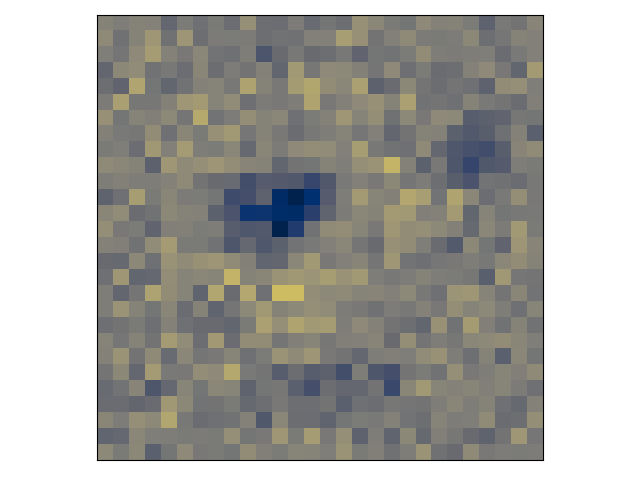}
 & 
\includegraphics[width=0.07\linewidth]{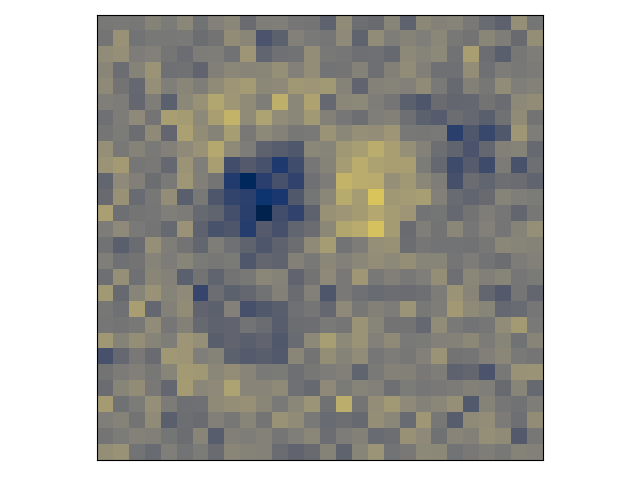}
 & 
\includegraphics[width=0.07\linewidth]{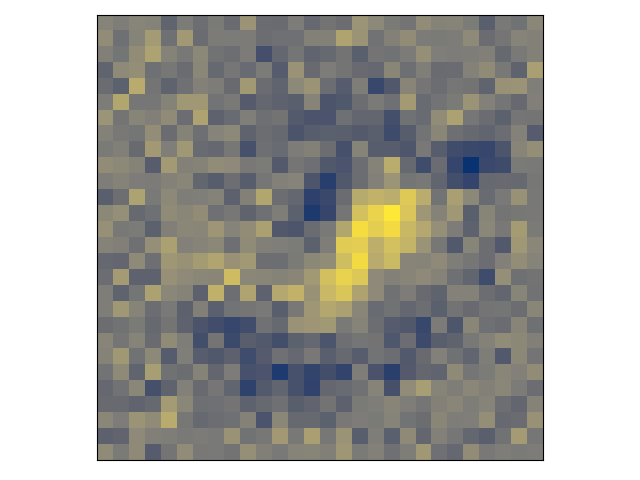}
 & 
 & 
 & 
 & 
 & 
\\
\midrule
\multirow{2}{*}{6}  & 
\includegraphics[width=0.07\linewidth]{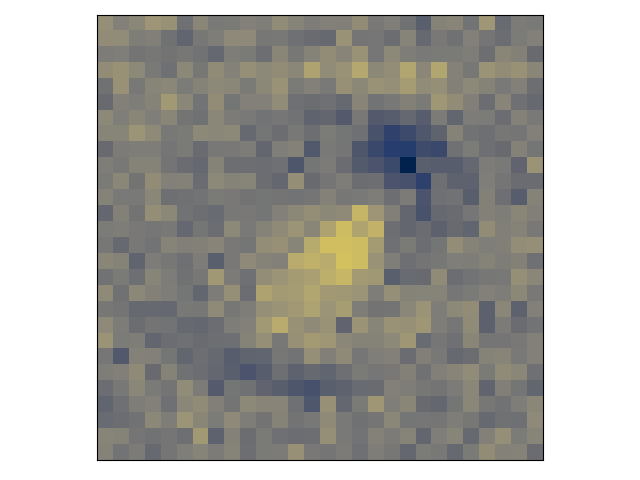}
 & 
\includegraphics[width=0.07\linewidth]{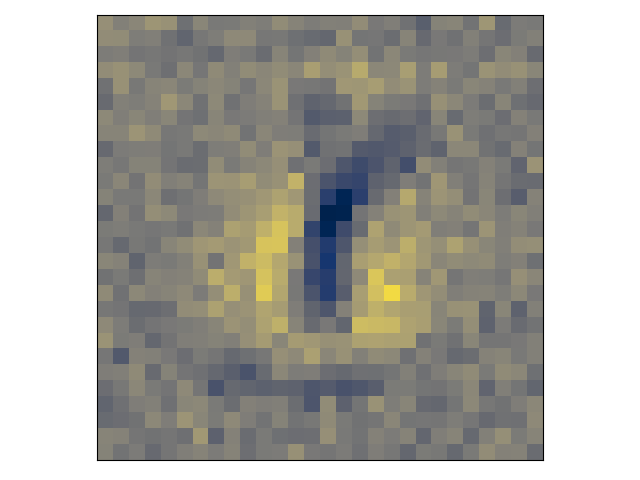}
 & 
\includegraphics[width=0.07\linewidth]{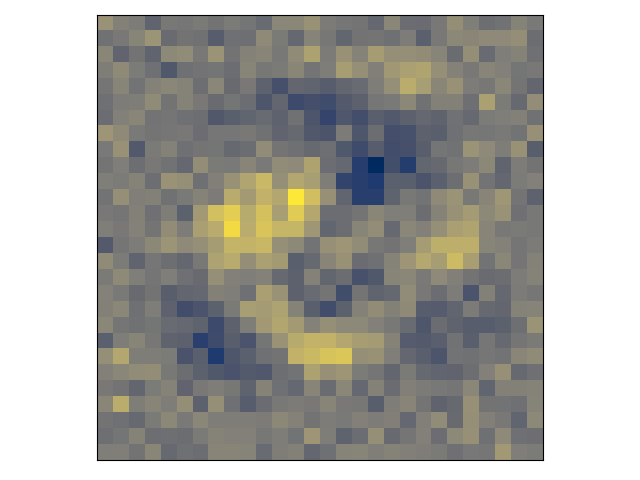}
 & 
\includegraphics[width=0.07\linewidth]{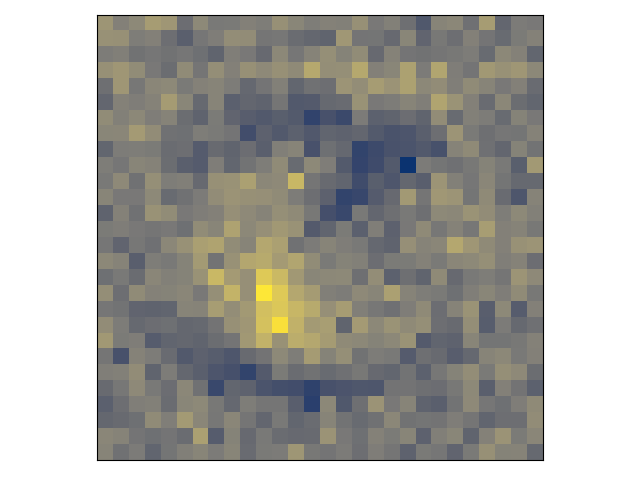}
 & 
\includegraphics[width=0.07\linewidth]{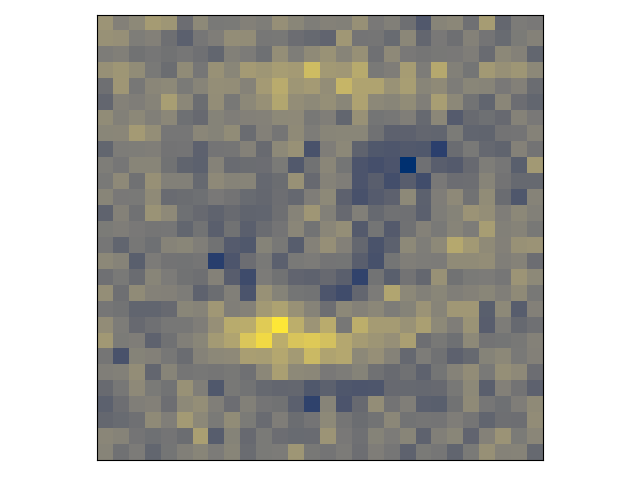}
 & 
\includegraphics[width=0.07\linewidth]{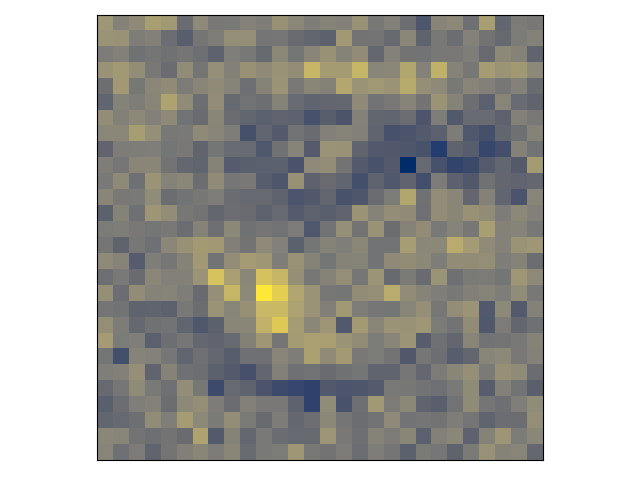}
 & 
 & 
 & 
 & 
\\
 & 
\includegraphics[width=0.07\linewidth]{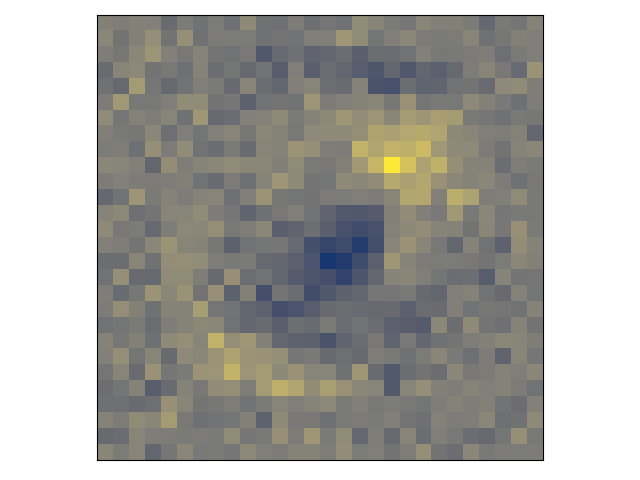}
 & 
\includegraphics[width=0.07\linewidth]{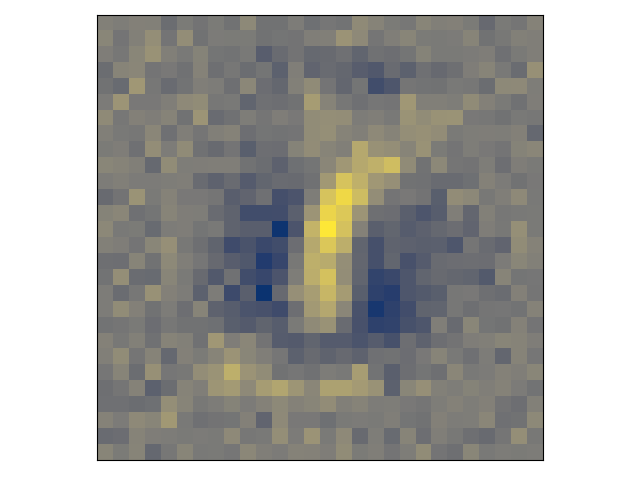}
 & 
\includegraphics[width=0.07\linewidth]{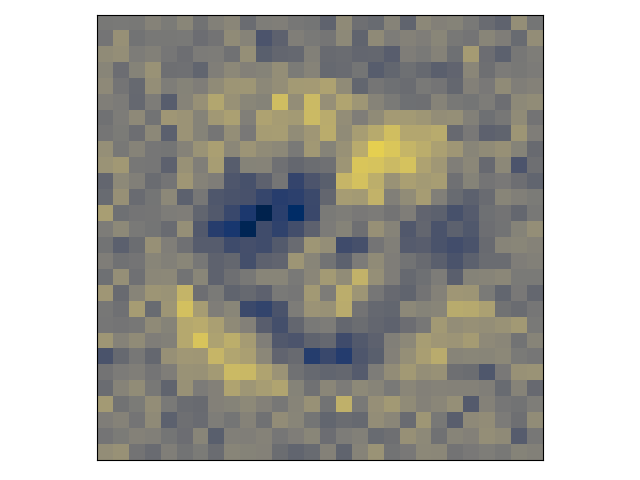}
 & 
\includegraphics[width=0.07\linewidth]{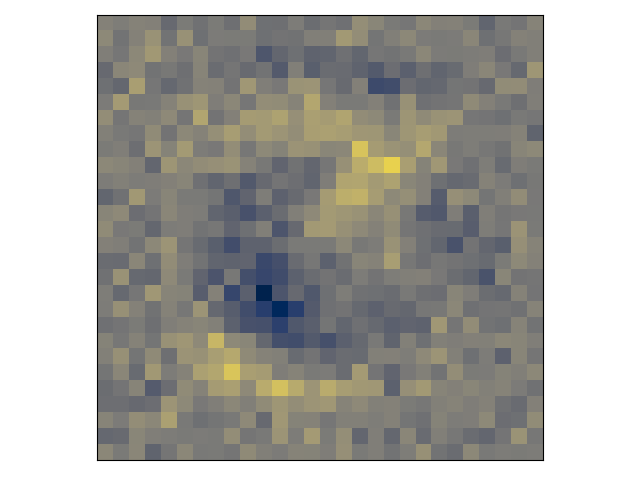}
 & 
\includegraphics[width=0.07\linewidth]{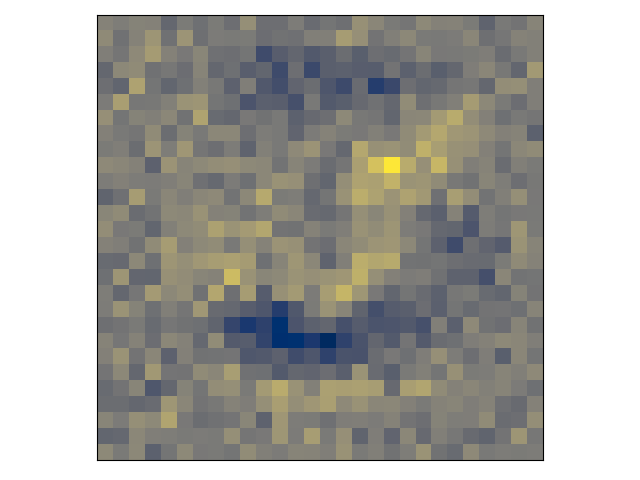}
 & 
\includegraphics[width=0.07\linewidth]{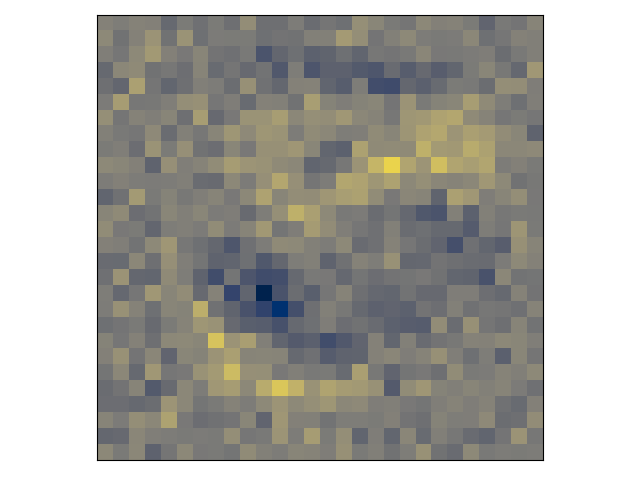}
 & 
 & 
 & 
 & 
\\
\midrule
\multirow{2}{*}{7}  & 
\includegraphics[width=0.07\linewidth]{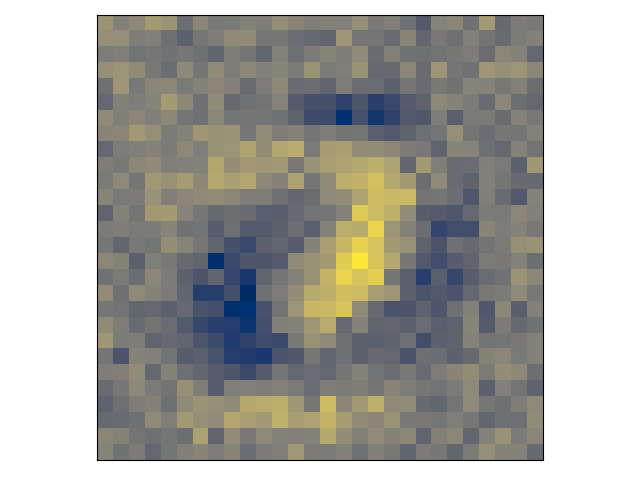}
 & 
\includegraphics[width=0.07\linewidth]{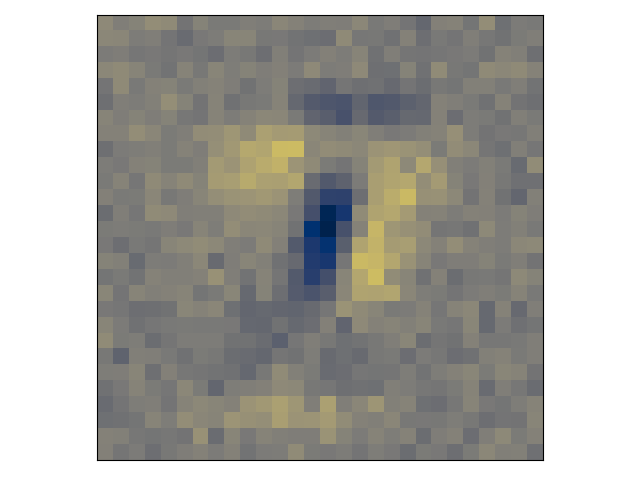}
 & 
\includegraphics[width=0.07\linewidth]{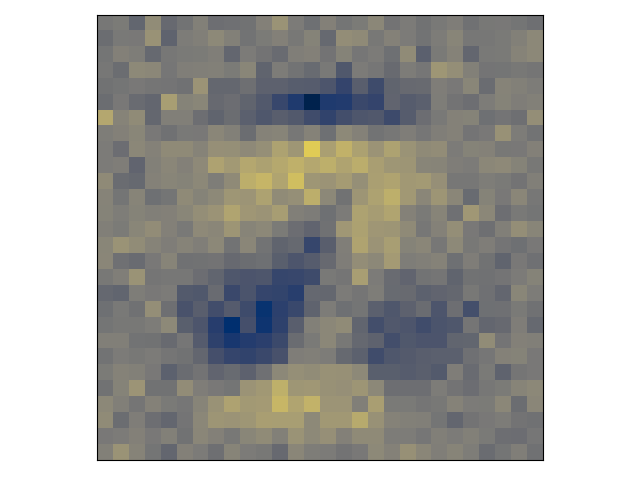}
 & 
\includegraphics[width=0.07\linewidth]{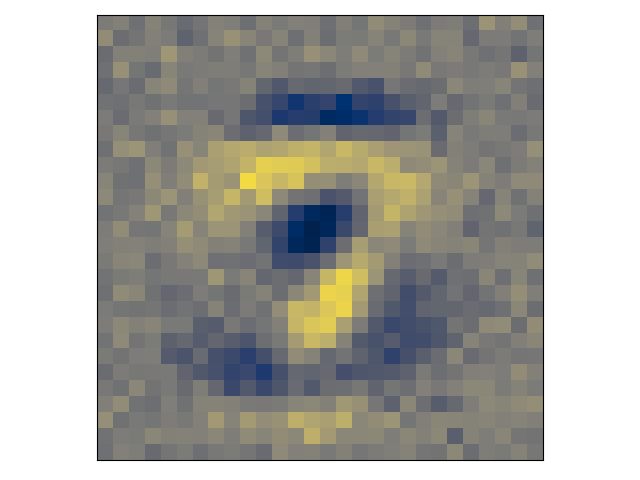}
 & 
\includegraphics[width=0.07\linewidth]{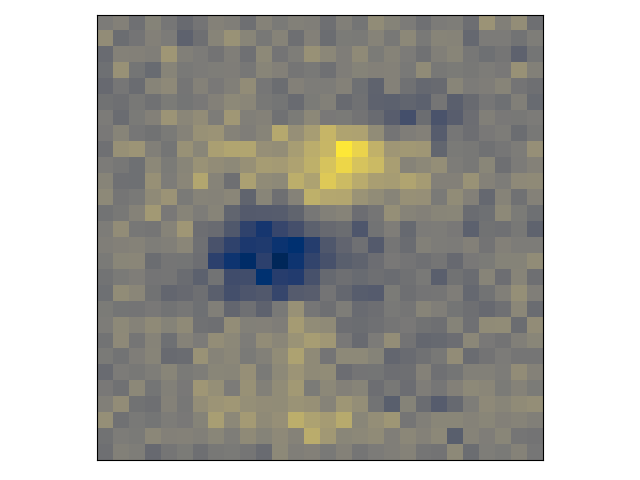}
 & 
\includegraphics[width=0.07\linewidth]{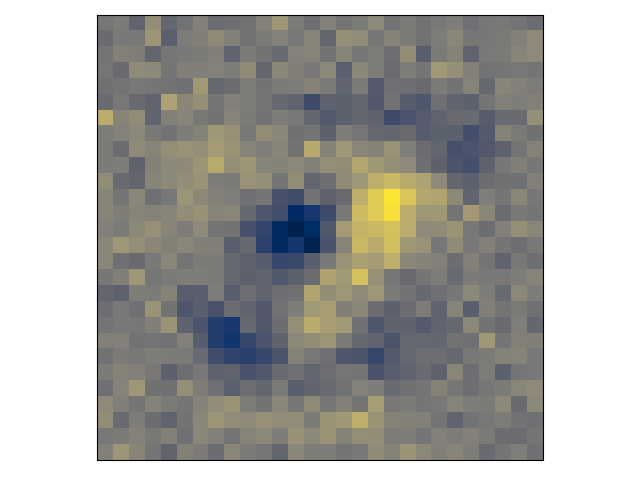}
 & 
\includegraphics[width=0.07\linewidth]{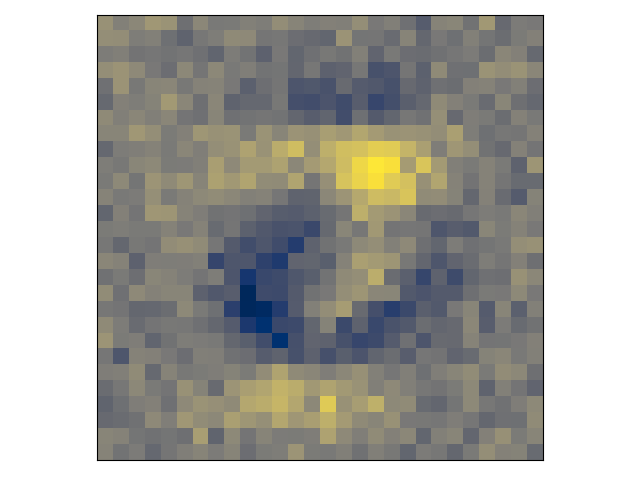}
 & 
 & 
 & 
\\
 & 
\includegraphics[width=0.07\linewidth]{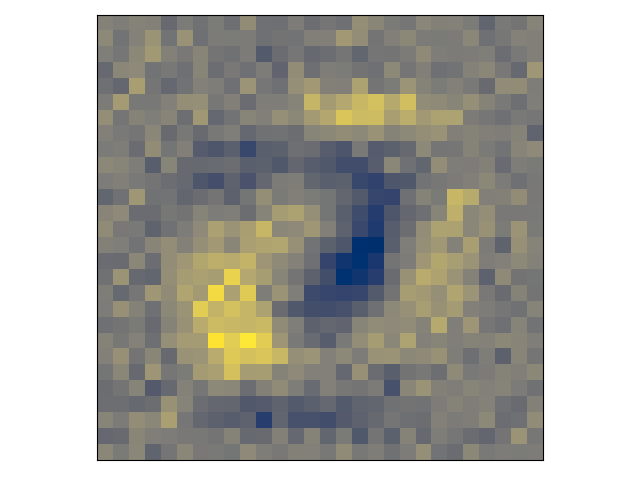}
 & 
\includegraphics[width=0.07\linewidth]{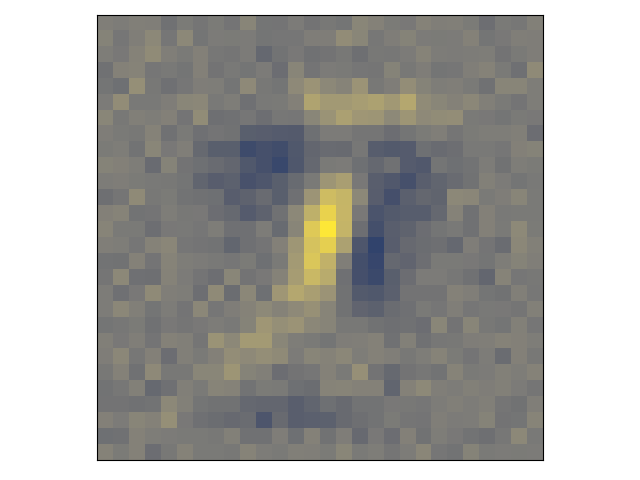}
 & 
\includegraphics[width=0.07\linewidth]{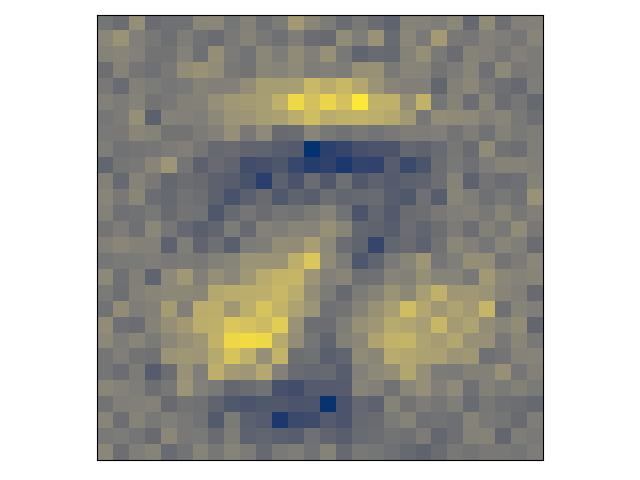}
 & 
\includegraphics[width=0.07\linewidth]{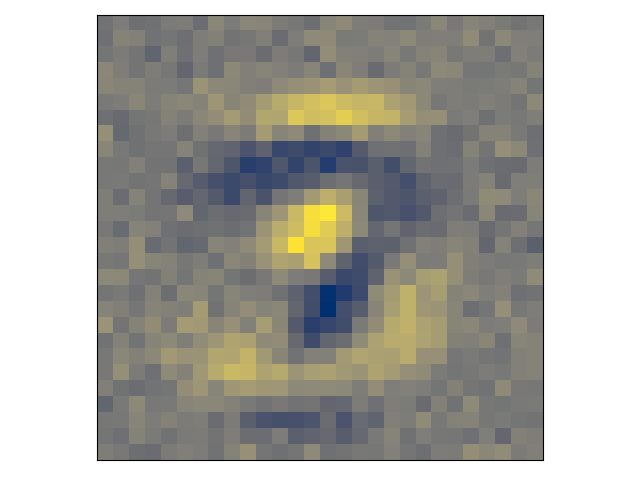}
 & 
\includegraphics[width=0.07\linewidth]{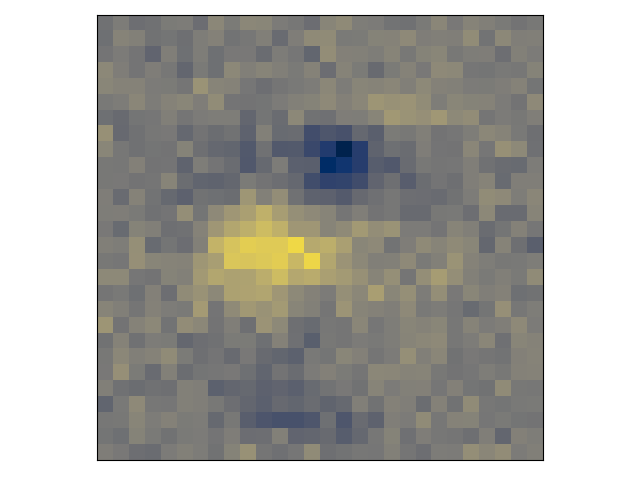}
 & 
\includegraphics[width=0.07\linewidth]{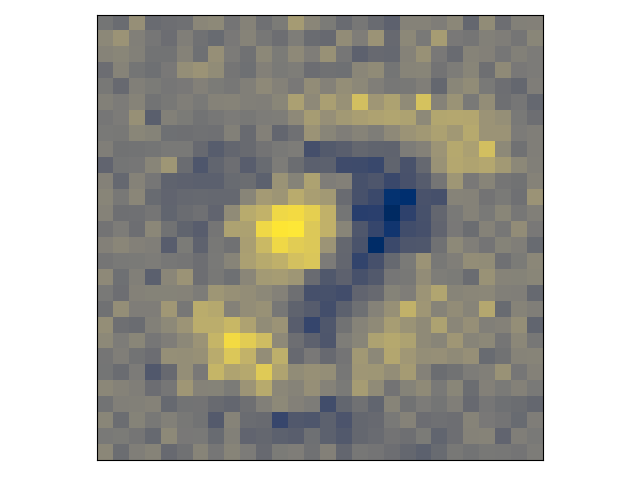}
 & 
\includegraphics[width=0.07\linewidth]{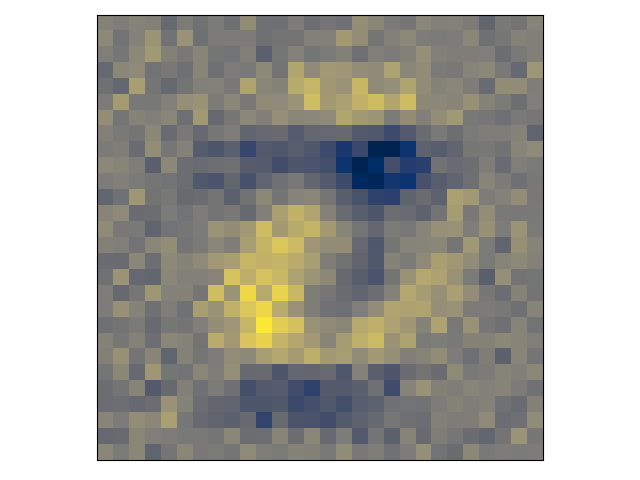}
 & 
 & 
 & 
\\
\midrule
\multirow{2}{*}{8}  & 
\includegraphics[width=0.07\linewidth]{mnist_basis_elements/ovo/0-8_u0.png}
 & 
\includegraphics[width=0.07\linewidth]{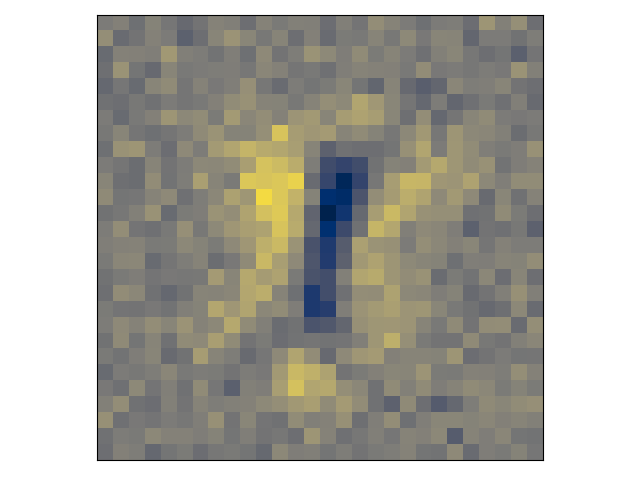}
 & 
\includegraphics[width=0.07\linewidth]{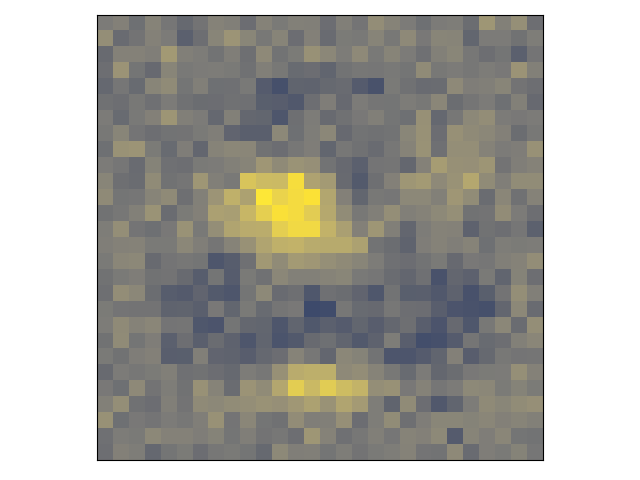}
 & 
\includegraphics[width=0.07\linewidth]{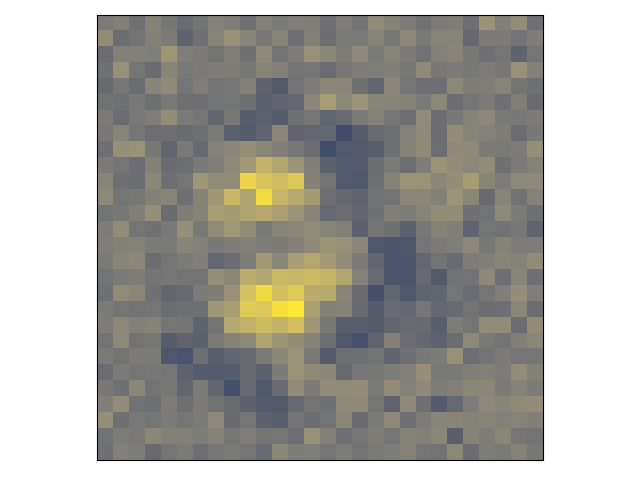}
 & 
\includegraphics[width=0.07\linewidth]{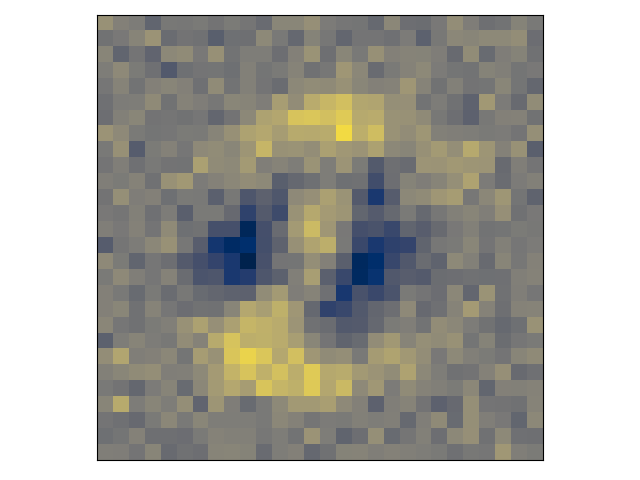}
 & 
\includegraphics[width=0.07\linewidth]{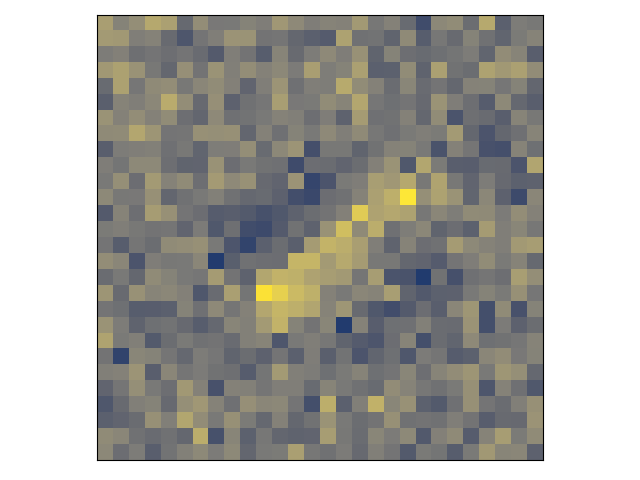}
 & 
\includegraphics[width=0.07\linewidth]{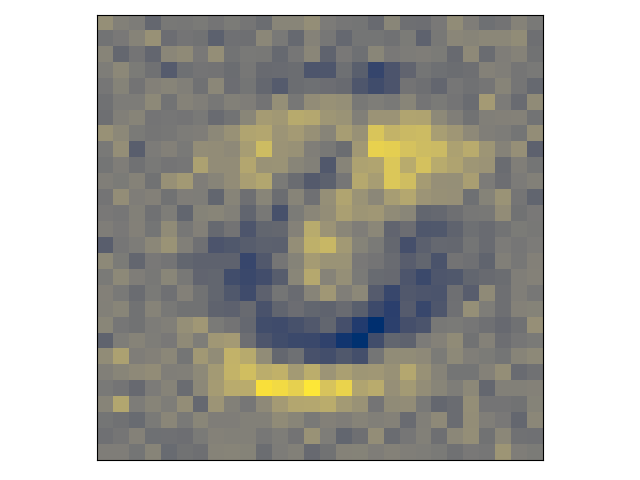}
 & 
\includegraphics[width=0.07\linewidth]{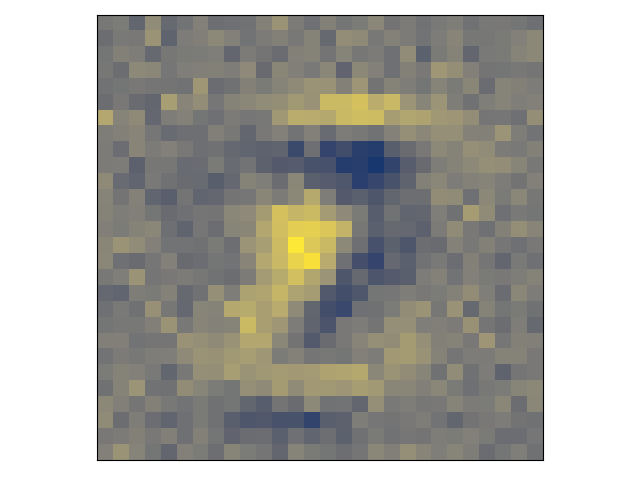}
 & 
 & 
\\
 & 
\includegraphics[width=0.07\linewidth]{mnist_basis_elements/ovo/0-8_u1.png}
 & 
\includegraphics[width=0.07\linewidth]{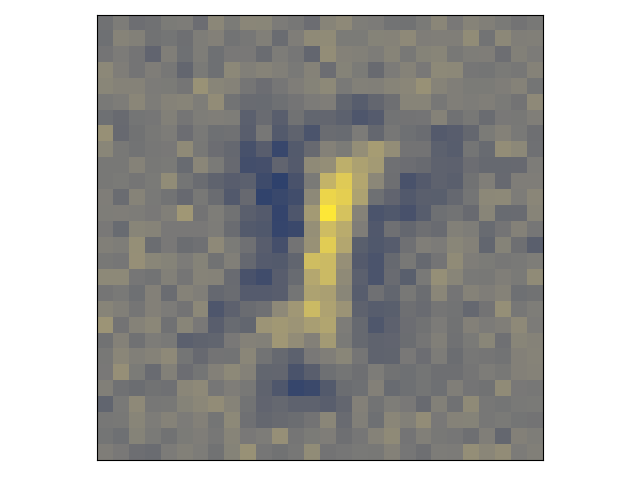}
 & 
\includegraphics[width=0.07\linewidth]{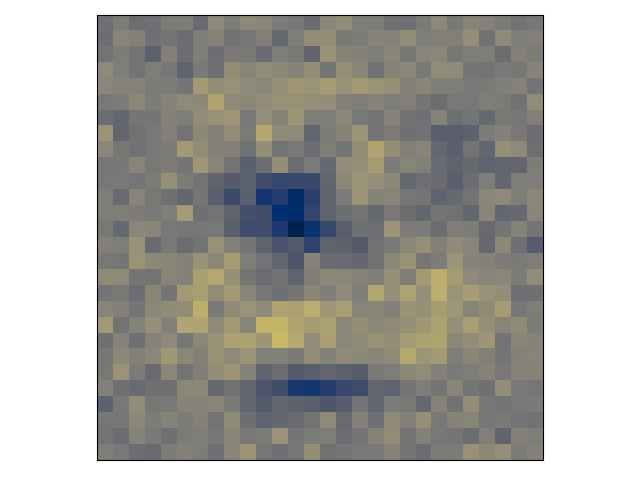}
 & 
\includegraphics[width=0.07\linewidth]{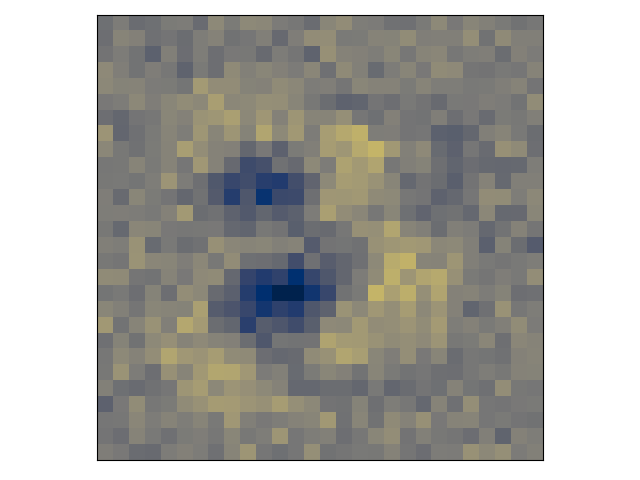}
 & 
\includegraphics[width=0.07\linewidth]{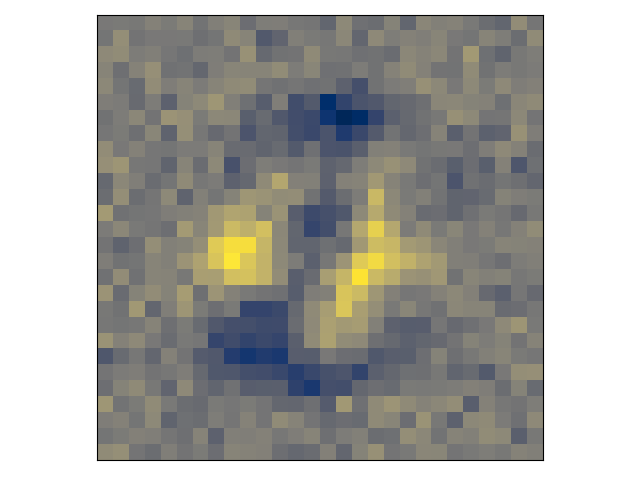}
 & 
\includegraphics[width=0.07\linewidth]{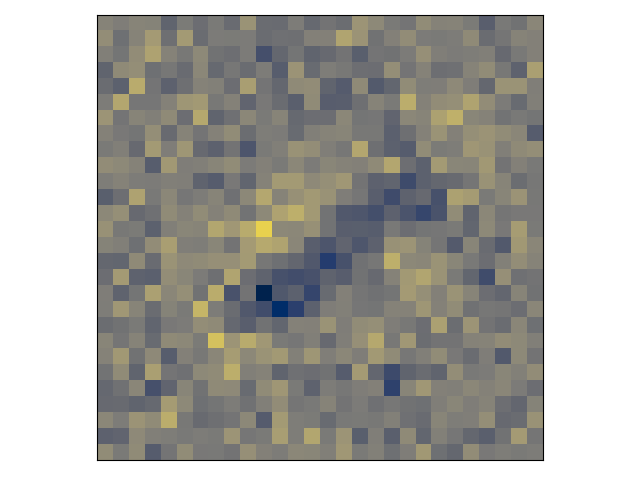}
 & 
\includegraphics[width=0.07\linewidth]{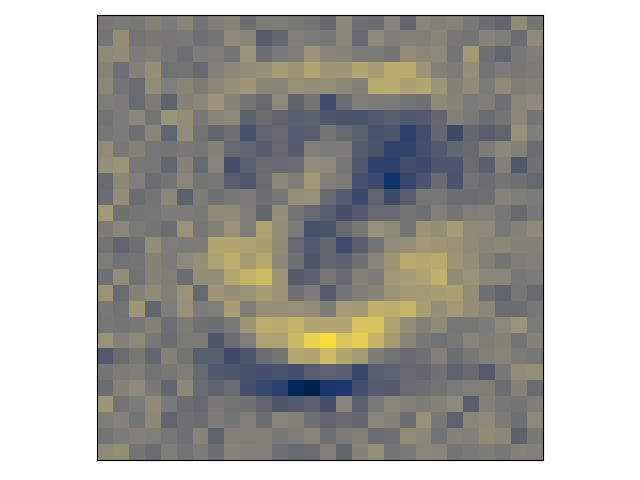}
 & 
\includegraphics[width=0.07\linewidth]{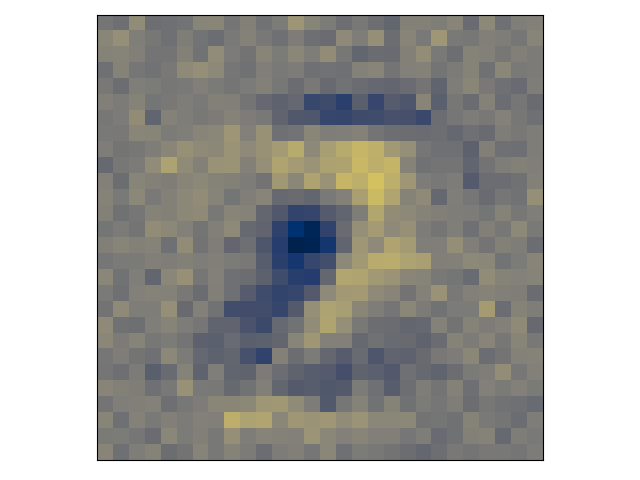}
 & 
 & 
\\
\midrule
\multirow{2}{*}{9}  & 
\includegraphics[width=0.07\linewidth]{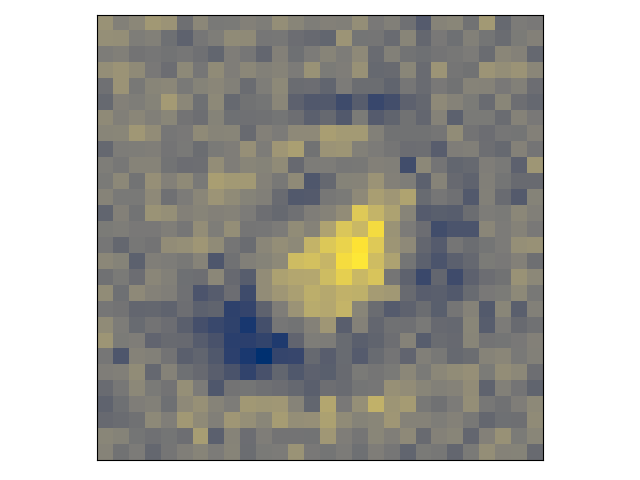}
 & 
\includegraphics[width=0.07\linewidth]{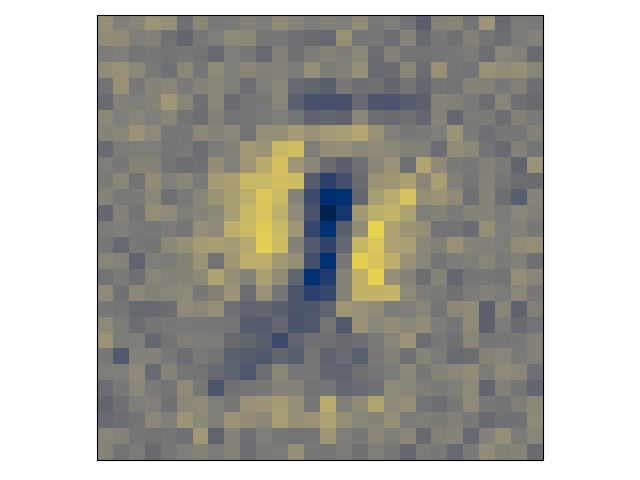}
 & 
\includegraphics[width=0.07\linewidth]{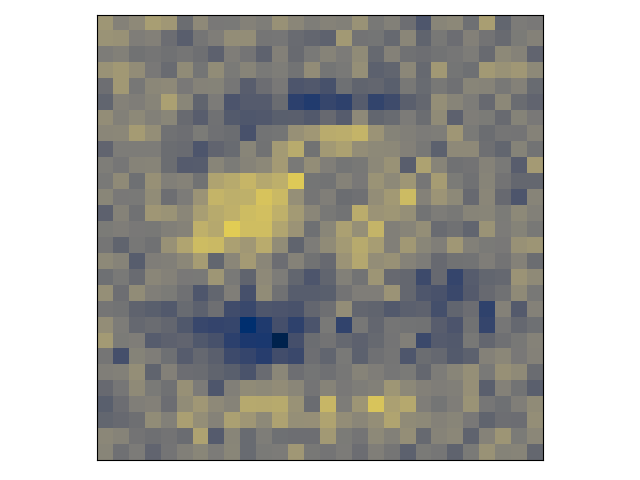}
 & 
\includegraphics[width=0.07\linewidth]{mnist_basis_elements/ovo/3-9_u0.png}
 & 
\includegraphics[width=0.07\linewidth]{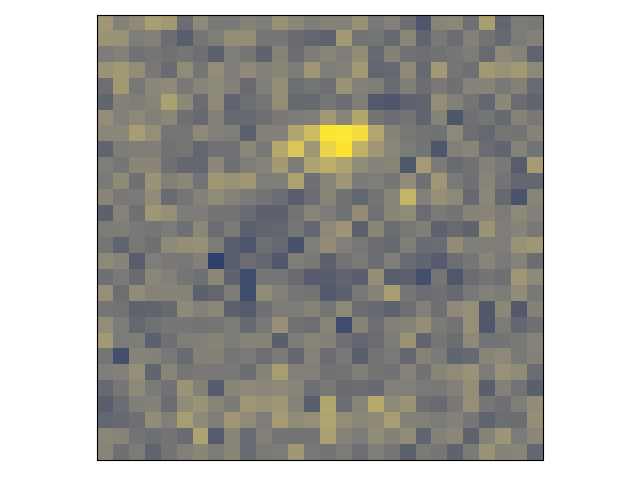}
 & 
\includegraphics[width=0.07\linewidth]{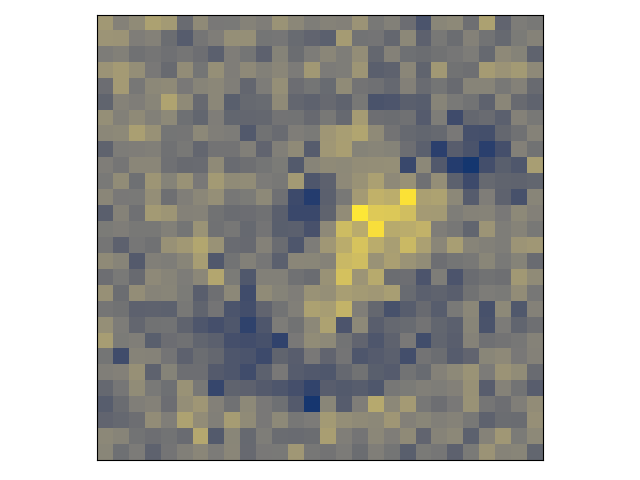}
 & 
\includegraphics[width=0.07\linewidth]{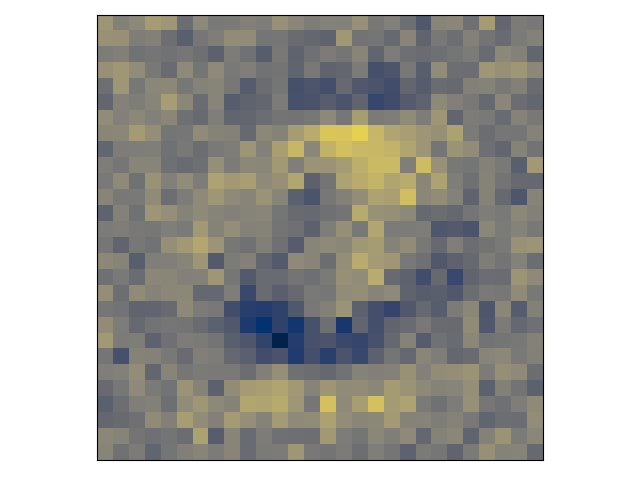}
 & 
\includegraphics[width=0.07\linewidth]{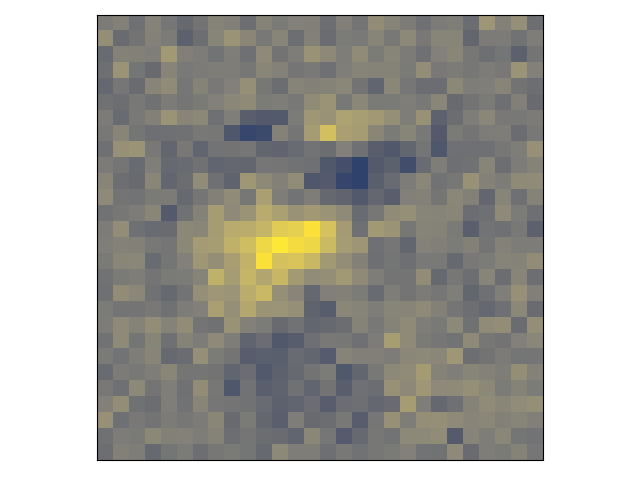}
 & 
\includegraphics[width=0.07\linewidth]{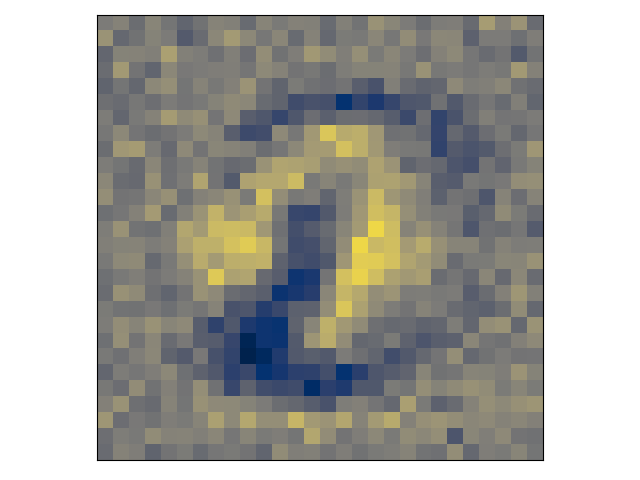}
 & 
\\
 & 
\includegraphics[width=0.07\linewidth]{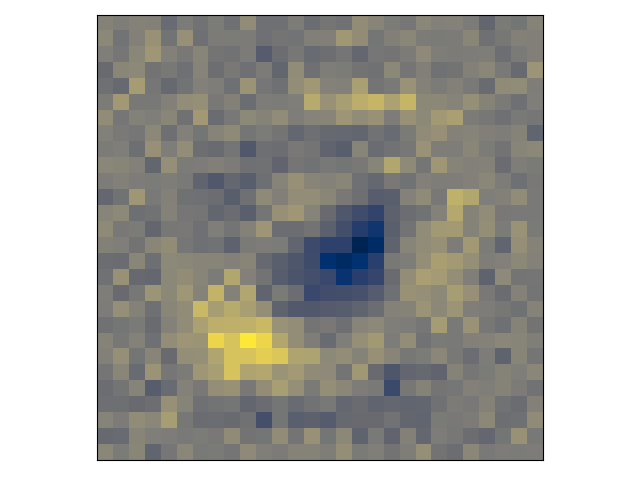}
 & 
\includegraphics[width=0.07\linewidth]{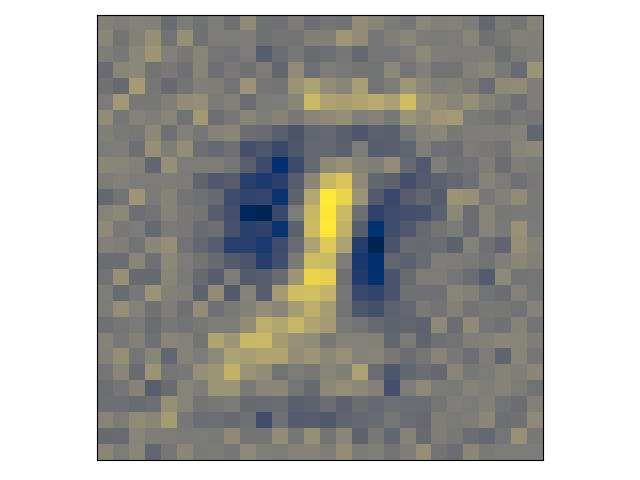}
 & 
\includegraphics[width=0.07\linewidth]{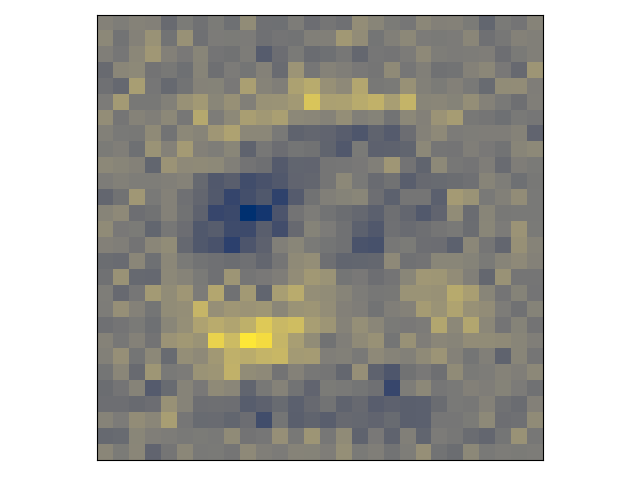}
 & 
\includegraphics[width=0.07\linewidth]{mnist_basis_elements/ovo/3-9_u1.png}
 & 
\includegraphics[width=0.07\linewidth]{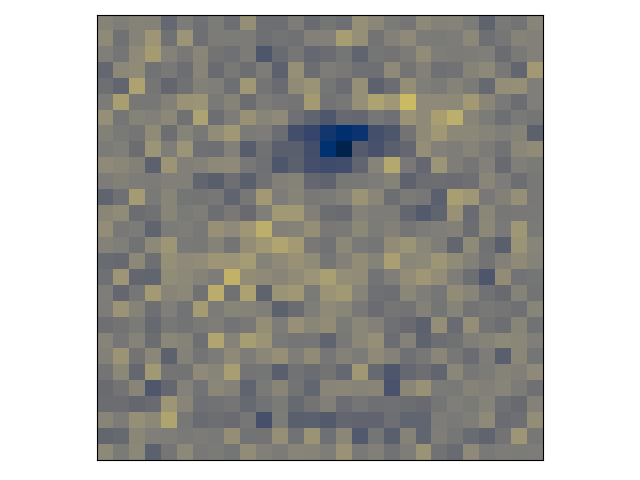}
 & 
\includegraphics[width=0.07\linewidth]{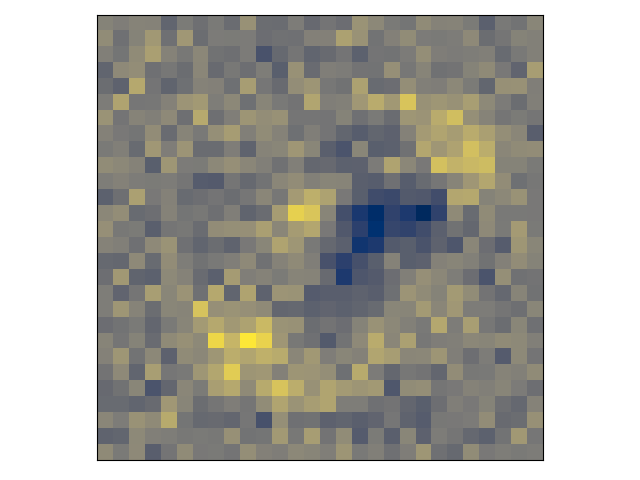}
 & 
\includegraphics[width=0.07\linewidth]{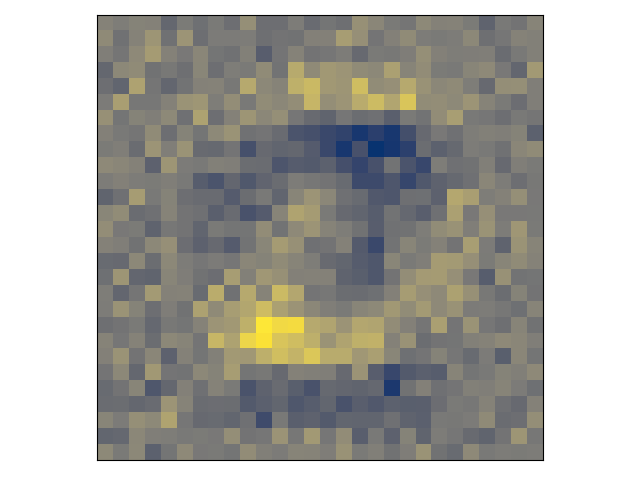}
 & 
\includegraphics[width=0.07\linewidth]{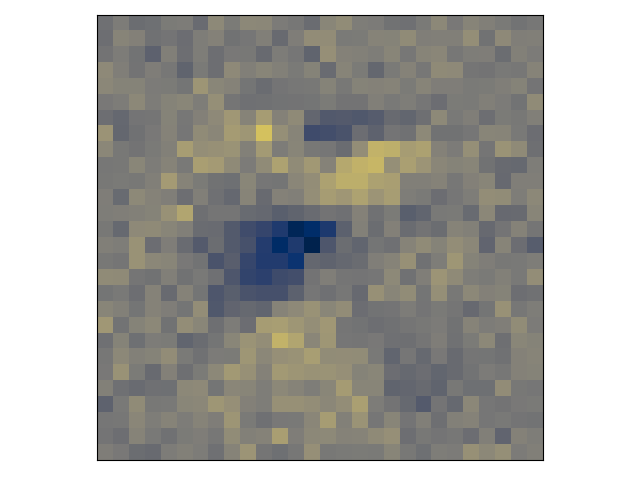}
 & 
\includegraphics[width=0.07\linewidth]{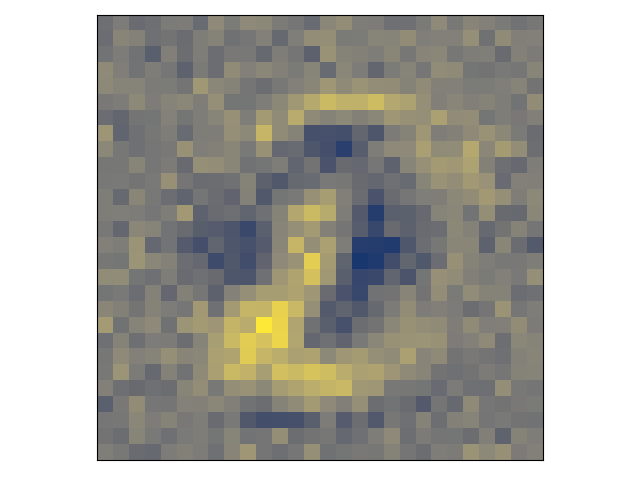}
 & 
\\
\midrule

    \end{tabular}
    \caption{All basis vectors found by \PKM{} applied to MNIST with OVO strategy.}
    \label{tab:mnist_all_ovo}
\end{table*}

%% file: ms.bbl
\begin{thebibliography}{38}
\providecommand{\natexlab}[1]{#1}
\providecommand{\url}[1]{{#1}}
\providecommand{\urlprefix}{URL }
\expandafter\ifx\csname urlstyle\endcsname\relax
  \providecommand{\doi}[1]{DOI~\discretionary{}{}{}#1}\else
  \providecommand{\doi}{DOI~\discretionary{}{}{}\begingroup
  \urlstyle{rm}\Url}\fi
\providecommand{\eprint}[2][]{\url{#2}}

\bibitem[{Bartlett and Mendelson(2002)}]{bartlett2002rademacher}
Bartlett PL, Mendelson S (2002) Rademacher and gaussian complexities: Risk
  bounds and structural results. Journal of Machine Learning Research
  3(Nov):463--482

\bibitem[{Blachnik and Kordos(2011)}]{blachnik2011simplifying}
Blachnik M, Kordos M (2011) Simplifying svm with weighted lvq algorithm. In:
  International Conference on Intelligent Data Engineering and Automated
  Learning, Springer, pp 212--219

\bibitem[{Broomhead and Lowe(1988)}]{broomhead1988radial}
Broomhead DS, Lowe D (1988) Radial basis functions, multi-variable functional
  interpolation and adaptive networks. Tech. rep., Royal Signals and Radar
  Establishment Malvern (United Kingdom)

\bibitem[{Burges(1996)}]{burges1996simplified}
Burges CJ (1996) Simplified support vector decision rules. In: Proc. 13th
  International Conference on Machine Learning, pp 71--77

\bibitem[{Cichonska et~al.(2018)Cichonska, Pahikkala, Szedmak, Julkunen,
  Airola, Heinonen, Aittokallio, and Rousu}]{cichonska2018learning}
Cichonska A, Pahikkala T, Szedmak S, Julkunen H, Airola A, Heinonen M,
  Aittokallio T, Rousu J (2018) Learning with multiple pairwise kernels for
  drug bioactivity prediction. Bioinformatics 34(13):i509--i518

\bibitem[{Davis et~al.(1994)Davis, Mallat, and Zhang}]{davis1994adaptive}
Davis GM, Mallat SG, Zhang Z (1994) Adaptive time-frequency decompositions.
  Optical engineering 33(7):2183--2192

\bibitem[{Downs et~al.(2001)Downs, Gates, and Masters}]{downs2001exact}
Downs T, Gates KE, Masters A (2001) Exact simplification of support vector
  solutions. Journal of Machine Learning Research 2(Dec):293--297

\bibitem[{Drineas and Mahoney(2005)}]{drineas2005nystrom}
Drineas P, Mahoney MW (2005) On the nystr{\"o}m method for approximating a gram
  matrix for improved kernel-based learning. journal of machine learning
  research 6(Dec):2153--2175

\bibitem[{Duchi et~al.(2008)Duchi, Shalev-Shwartz, Singer, and
  Chandra}]{duchi2008efficient}
Duchi J, Shalev-Shwartz S, Singer Y, Chandra T (2008) Efficient projections
  onto the l 1-ball for learning in high dimensions. In: Proceedings of the
  25th international conference on Machine learning, pp 272--279

\bibitem[{Geebelen et~al.(2012)Geebelen, Suykens, and
  Vandewalle}]{geebelen2012reducing}
Geebelen D, Suykens JA, Vandewalle J (2012) Reducing the number of support
  vectors of svm classifiers using the smoothed separable case approximation.
  IEEE Transactions on Neural Networks and Learning Systems 23(4):682--688

\bibitem[{G{\"o}nen and Alpayd{\i}n(2011)}]{gonen2011multiple}
G{\"o}nen M, Alpayd{\i}n E (2011) Multiple kernel learning algorithms. The
  Journal of Machine Learning Research 12:2211--2268

\bibitem[{Guyon and Elisseeff(2003)}]{guyon2003introduction}
Guyon I, Elisseeff A (2003) An introduction to variable and feature selection.
  Journal of machine learning research 3(Mar):1157--1182

\bibitem[{Hofmann et~al.(2008)Hofmann, Sch{\"o}lkopf, and
  Smola}]{hofmann2008kernel}
Hofmann T, Sch{\"o}lkopf B, Smola AJ (2008) Kernel methods in machine learning.
  The annals of statistics pp 1171--1220

\bibitem[{Huang et~al.(2010)Huang, Zhang et~al.}]{huang2010benefit}
Huang J, Zhang T, et~al. (2010) The benefit of group sparsity. The Annals of
  Statistics 38(4):1978--2004

\bibitem[{Jia et~al.(2010)Jia, Salzmann, Darrell et~al.}]{jia2010factorized}
Jia Y, Salzmann M, Darrell T, et~al. (2010) Factorized latent spaces with
  structured sparsity. In: NIPS, vol~10, pp 982--990

\bibitem[{Joachims and Yu(2009)}]{joachims2009sparse}
Joachims T, Yu CNJ (2009) Sparse kernel svms via cutting-plane training.
  Machine learning 76(2):179--193

\bibitem[{Keriven et~al.(2017)Keriven, Tremblay, Traonmilin, and
  Gribonval}]{keriven2017compressive}
Keriven N, Tremblay N, Traonmilin Y, Gribonval R (2017) Compressive k-means.
  In: 2017 IEEE International Conference on Acoustics, Speech and Signal
  Processing (ICASSP), IEEE, pp 6369--6373

\bibitem[{Keriven et~al.(2018)Keriven, Bourrier, Gribonval, and
  P{\'e}rez}]{keriven2018sketching}
Keriven N, Bourrier A, Gribonval R, P{\'e}rez P (2018) Sketching for
  large-scale learning of mixture models. Information and Inference: A Journal
  of the IMA 7(3):447--508

\bibitem[{Khaledi et~al.(2020)Khaledi, Weimann, Schniederjans, Asgari, Kuo,
  Oliver, Cabot, Kola, Gastmeier, Hogardt et~al.}]{khaledi2020predicting}
Khaledi A, Weimann A, Schniederjans M, Asgari E, Kuo TH, Oliver A, Cabot G,
  Kola A, Gastmeier P, Hogardt M, et~al. (2020) Predicting antimicrobial
  resistance in pseudomonas aeruginosa with machine learning-enabled molecular
  diagnostics. EMBO Molecular Medicine 12(3)

\bibitem[{Kim et~al.(2016)Kim, Khanna, and Koyejo}]{kim2016examples}
Kim B, Khanna R, Koyejo OO (2016) Examples are not enough, learn to criticize!
  criticism for interpretability. Advances in neural information processing
  systems 29:2280--2288

\bibitem[{Koltchinskii et~al.(2002)Koltchinskii, Panchenko
  et~al.}]{koltchinskii2002empirical}
Koltchinskii V, Panchenko D, et~al. (2002) Empirical margin distributions and
  bounding the generalization error of combined classifiers. Annals of
  statistics 30(1):1--50

\bibitem[{Krzyzak and Linder(1998)}]{krzyzak1998radial}
Krzyzak A, Linder T (1998) Radial basis function networks and complexity
  regularization in function learning. IEEE Transactions on Neural Networks
  9(2):247--256

\bibitem[{Lei et~al.(2014)Lei, Ding, and Zhang}]{lei2014generalization}
Lei Y, Ding L, Zhang W (2014) Generalization performance of radial basis
  function networks. IEEE transactions on neural networks and learning systems
  26(3):551--564

\bibitem[{Meanti et~al.(2020)Meanti, Carratino, Rosasco, and
  Rudi}]{meanti2020kernel}
Meanti G, Carratino L, Rosasco L, Rudi A (2020) Kernel methods through the
  roof: Handling billions of points efficiently. Advances in Neural Information
  Processing Systems 33

\bibitem[{Mohri et~al.(2018)Mohri, Rostamizadeh, and
  Talwalkar}]{mohri2018foundations}
Mohri M, Rostamizadeh A, Talwalkar A (2018) Foundations of machine learning.
  MIT press

\bibitem[{Neto and Barreto(2013)}]{neto2013opposite}
Neto AR, Barreto GA (2013) Opposite maps: Vector quantization algorithms for
  building reduced-set svm and lssvm classifiers. Neural processing letters
  37(1):3--19

\bibitem[{Niyogi and Girosi(1996)}]{niyogi1996relationship}
Niyogi P, Girosi F (1996) On the relationship between generalization error,
  hypothesis complexity, and sample complexity for radial basis functions.
  Neural Computation 8(4):819--842

\bibitem[{Panja and Pal(2018)}]{panja2018ms}
Panja R, Pal NR (2018) Ms-svm: minimally spanned support vector machine.
  Applied Soft Computing 64:356--365

\bibitem[{Pati et~al.(1993)Pati, Rezaiifar, and
  Krishnaprasad}]{pati1993orthogonal}
Pati YC, Rezaiifar R, Krishnaprasad PS (1993) Orthogonal matching pursuit:
  Recursive function approximation with applications to wavelet decomposition.
  In: Proceedings of 27th Asilomar conference on signals, systems and
  computers, IEEE, pp 40--44

\bibitem[{Rahimi and Recht(2008)}]{rahimi2008random}
Rahimi A, Recht B (2008) Random features for large-scale kernel machines. In:
  NIPS, pp 1177--1184

\bibitem[{Rudi et~al.(2017)Rudi, Carratino, and Rosasco}]{rudi2017falkon}
Rudi A, Carratino L, Rosasco L (2017) Falkon: an optimal large scale kernel
  method. In: Proceedings of the 31st International Conference on Neural
  Information Processing Systems, pp 3891--3901

\bibitem[{Sch{\"o}lkopf et~al.(2001)Sch{\"o}lkopf, Herbrich, and
  Smola}]{scholkopf2001generalized}
Sch{\"o}lkopf B, Herbrich R, Smola AJ (2001) A generalized representer theorem.
  In: International conference on computational learning theory, Springer, pp
  416--426

\bibitem[{Schwenker et~al.(2001)Schwenker, Kestler, and
  Palm}]{schwenker2001three}
Schwenker F, Kestler HA, Palm G (2001) Three learning phases for
  radial-basis-function networks. Neural networks 14(4-5):439--458

\bibitem[{Uurtio et~al.(2019)Uurtio, Bhadra, and Rousu}]{uurtio2019large}
Uurtio V, Bhadra S, Rousu J (2019) Large-scale sparse kernel canonical
  correlation analysis. In: International Conference on Machine Learning, pp
  6383--6391

\bibitem[{Vincent and Bengio(2002)}]{vincent2002kernel}
Vincent P, Bengio Y (2002) Kernel matching pursuit. Machine learning
  48(1-3):165--187

\bibitem[{Wang et~al.(2020)Wang, Wang, Huang, and Du}]{wang2020condensing}
Wang X, Wang S, Huang Z, Du Y (2020) Condensing the solution of support vector
  machines via radius-margin bound. Applied Soft Computing p 107071

\bibitem[{Williams and Seeger(2001)}]{williams2001using}
Williams CKI, Seeger M (2001) Using the nystr{\"o}m method to speed up kernel
  machines. In: Advances in Neural Information Processing Systems~(NIPS), pp
  682--688

\bibitem[{Zhao et~al.(2010)Zhao, Morstatter, Sharma, Alelyani, Anand, and
  Liu}]{zhao2010advancing}
Zhao Z, Morstatter F, Sharma S, Alelyani S, Anand A, Liu H (2010) Advancing
  feature selection research. ASU feature selection repository pp 1--28

\end{thebibliography}
